\documentclass[11pt,draftclsnofoot,onecolumn]{IEEEtran}
\usepackage[utf8]{inputenc} % allow utf-8 input
\usepackage[T1]{fontenc}    % use 8-bit T1 fonts
\usepackage{hyperref}       % hyperlinks
\usepackage{url}            % simple URL typesetting
\usepackage{booktabs}       % professional-quality tables
\usepackage{amsfonts}       % blackboard math symbols
\usepackage{nicefrac}       % compact symbols for 1/2, etc.
\usepackage{microtype}      % microtypography
\usepackage{amsmath,dsfont}
\usepackage{epsfig}
\usepackage{algorithm,footnote}
\usepackage{algpseudocode}
\usepackage{syntonly,bm}
\usepackage[utf8]{inputenc}
\usepackage[english]{babel}
\usepackage{color}
\newtheorem{theorem}{Theorem}
\newtheorem{lemma}[theorem]{Lemma}
\allowdisplaybreaks[2]

\input{mysymbol.sty}

 \long\def\symbolfootnote[#1]#2{\begingroup
 	\def\thefootnote{\fnsymbol{footnote}}
 	\footnote[#1]{#2}\endgroup} \psfull

\title{Online Graph-Adaptive Learning \\with Scalability and Privacy}

\author{ Yanning Shen$^\ast$, \textit{Student Member}, \textit{IEEE},  Geert Leus$^\dag$, \textit{Fellow, IEEE}, \\and  Georgios~B.~Giannakis$^\ast$, \textit{Fellow, IEEE}}

\begin{document}

\maketitle

\symbolfootnote[0]{ Work in this paper was supported by grants NSF 1711471, 1500713 and NIH 1R01GM104975-01.}

\symbolfootnote[0]{$^\ast$Y. Shen and G. B. Giannakis are with the Dept.
	of ECE and the Digital Technology Center, University of
	Minnesota, 117 Pleasant Str. SE, Minneapolis, MN 55455. Tel: (612)625-4287; Emails:
	\texttt{\{shenx513,georgios\}@umn.edu } }
	
\symbolfootnote[0]{$^\dag$ G. Leus is with the Electrical Engineering, Mathematics and Computer Science,
Delft University of Technology, CD, Delft 2826, The Netherlands E-mail:,
\texttt{g.j.t.leus@tudelft.nl}}

\vspace{-12mm}
\begin{abstract}
Graphs are widely adopted for modeling complex systems, including financial, biological, and social networks. Nodes in networks usually entail  attributes, such as the age or gender of users in a social network. However, real-world networks can have very large size, and nodal attributes can be unavailable to a number of nodes, e.g., due to privacy concerns. Moreover, new nodes can emerge over time, which can necessitate real-time evaluation of their nodal attributes. In this context, the present paper deals with scalable learning of nodal attributes by estimating a nodal function based on noisy observations at a subset of nodes. A multikernel-based approach is developed which is scalable to large-size networks. Unlike most existing methods that re-solve the function estimation problem over all existing nodes whenever a new node joins the network, the novel method is capable of providing real-time evaluation of the function values on newly-joining nodes without resorting to a batch solver. Interestingly,  the novel scheme only relies on an encrypted version of  each node's connectivity in order to learn the nodal attributes, which promotes privacy. Experiments on both synthetic and real datasets corroborate the effectiveness of the proposed methods.
\end{abstract}

\section{Introduction}

Estimating nodal functions/signals over networks is a task emerging in various domains, such as social, brain, and power networks, to name a few. Functions of nodes can represent certain attributes or classes of these nodes. In Facebook for instance, each node represents a person, and the presence of an edge indicates that two persons are friends, while nodal attributes can be  age, gender or movie ratings of each person. In financial networks, where each node is a company, with links denoting trade between two companies, the function of the node can represent the category that each company belongs to, e.g., technology-, fashion-, or education-related.

%\textcolor{red}{I feel a paragraph is required here that introduces the interpolation problem in those networks. And then you can continue with the state-of-the-art.}

\textcolor{black}{In real-world networks, there are often unavailable nodal function values, due to, e.g., privacy issues. Hence, a topic of great practical importance
is to interpolate missing nodal values (class, ranking or function), based on the function values at a subset of observed nodes. Interpolation of nodal function values often relies on the assumption of ``smoothness'' over the graphs, which implies that neighboring nodes  will have similar nodal function values. For example, in social networks, people tend to rate e.g., movies similar to their friends, and in financial networks, companies that trade with each other usually belong to the same category. 
From this point of view, }
function estimation over graphs based on partial observations has been investigated extensively,  \cite{kolaczyk2009statistical,kondor2002diffusion,belkin2006manifold,wasserman2008statistical,lu2003link,giannakis2018pieee}. 
Function estimation has been also pursued in the context of semi-supervised learning, e.g., for transductive regression or classification, see e.g., \cite{chapelle2009semi,chapelle2000transductive,cortes2007transductive,berberidis2018adaptive}. The same task has been studied recently  as signal reconstruction over graphs, see e.g., \cite{narang2013signal,wang2015local,romero2017kernel,marques2016sampling,shuman2013emerging}, where  signal values on  unobserved nodes can be estimated by properly introducing a graph-aware prior.
%
%,isufi2017autoregressive,marques2016sampling
%
Kernel-based methods for learning over graphs offer a unifying framework that includes linear and nonlinear function estimators \cite{romero2017kernel,smola2003kernels,ioannidis2018inference}. The nonlinear methods outperform the linear ones but suffer from the curse of dimensionality \cite{wahba1990spline}, rendering them less attractive for large-scale networks.   

To alleviate this limitation, a scalable  kernel-based approach will be introduced in the present paper, which leverages the random feature approximation to ensure \emph{scalability} while also allowing \emph{real-time} evaluation of the functions over large-scale dynamic networks.  In addition, the novel approach incorporates a data-driven scheme for \emph{adaptive} kernel selection.

Adaptive learning over graphs has been also investigated for tracking and learning over possibly dynamic networks, e.g., \cite{di2018adaptive,ioannidis2018inference}. Least mean-squares and recursive least-squares adaptive schemes have been developed in \cite{di2018adaptive}, without explicitly accounting for evolving network topologies. In contrast, \cite{ioannidis2018inference} proposed a
kernel-based reconstruction scheme to track time-varying
signals over  time-evolving topologies, but assumed that the kernel function is selected a priori. All these prior works assume that the network size is fixed.

In certain applications however, new nodes may join the network over time. For example, hundreds of new users are joining Facebook or Netflix every day, and new companies are founded in financial networks regularly. Real-time and scalable estimation of the desired functions on these newly-joining nodes is of great importance. While simple schemes such as averaging over one- or multi-hop neighborhoods are scalable to network size by predicting the value  on each newly-coming node as a weighted combination of its multi-hop neighborhoods~\cite{altman1992introduction}, they do 
%While simple schemes such as kNN \cite{altman1992introduction} are capable of handling this issue by predicting the value  on each newly-coming node as a weighted combination of its $k$-nearest neighbors, they do  
not capture global information over the network. In addition, existing rigorous approaches are in general less efficient in accounting for newly-joining nodes, and need to solve the  problem over all nodes, every time  new nodes join the network, which incurs complexity $\mathcal{O}(N^3)$, where $N$ denotes the network size~\cite{romero2017kernel,smola2003kernels}. As a result, these methods are not amenable to real-time evaluation over newly-joining nodes. To this end, the present paper develops a scalable \emph{online graph-adaptive} algorithm that can efficiently estimate nodal functions on  newly-joining nodes `on the fly.'

Besides scalability and adaptivity, nodes may have firm \emph{privacy} requirements, and may therefore not be willing to reveal who their neighbors are. However, most graph-based learning methods require knowing the entire connectivity pattern, and thus cannot meet the privacy requirements. The novel random feature based approach on the other hand, only requires an encrypted version of each node's connectivity pattern, which makes it appealing for  networks with stringent privacy constraints.

In short, we put forth a novel online multikernel learning (MKL) framework for effectively learning and tracking nonlinear functions over graphs. Our contributions are as follows.

\noindent\textbf{c1)} A scalable MKL approach is developed to efficiently estimate the nodal function values both on the observed and un-observed nodes of a graph;\\
\noindent\textbf{c2)} The resultant algorithm is capable of estimating the function value of newly incoming nodes with high accuracy without having to solve the batch problem over all nodes, making it highly scalable as the network size grows, and suitable for nodal function estimation in dynamic networks; 
\\
\noindent\textbf{c3)} Unlike most existing methods that rely on nodal feature vectors in order to learn the function, the proposed scheme simply capitalizes on the connectivity pattern of each node, while at the same time, nodal feature vectors can be easily incorporated if available; and,\\
\noindent\textbf{c4)} The proposed algorithm does not require nodes  to share connectivity patterns. Instead, a privacy-preserving scheme is developed for estimating the nodal function values based on an encrypted version of the nodal connectivity patterns, hence respecting node privacy.

%\noindent\textbf{c4)} The novel scheme is capable of adaptively selecting the set of features that most fits the learning task if multiple sets of features are available from different sources.

The rest of this paper is organized as follows. Preliminaries are  in Section~\ref{sec:pre}, while Section~\ref{sec:online} presents an online kernel-based algorithm that allows sequential processing of nodal samples. Section~\ref{sec:gradraker} develops an online MKL scheme for sequential data-driven kernel selection, which allows graph-adaptive selection of kernel functions to best fit the learning task of interest.
% In Section~\ref{sec:fadraker}, a generalization is provided to scenarios where multiple set of features are available, and a feature-adaptive scheme is developed. 
Finally, results of corroborating numerical tests on both synthetic and real data are presented in Section~\ref{sec:test}, while concluding remarks along with a discussion of ongoing and future directions are given in Section~\ref{sec:con}.

\noindent\textit{Notation}. Bold uppercase (lowercase) letters denote matrices (column vectors), while  $(\cdot)^{\top}$ and $\lambda_i(.)$ stand for matrix transposition, and the $i$th leading eigenvalue of the matrix argument, respectively. The identity matrix will be represented by $\mathbf{I}$, while $\mathbf{0}$ will denote the matrix of all zeros, and their dimensions will be clear from the context. Finally, the $\ell_p$ and Frobenius norms will be denoted by $\|\cdot\|_p$, and $\|\cdot\|_F$, respectively. 

%\noindent\textbf{Notation}. Bold uppercase (lowercase) letters will denote
%matrices (column vectors), while $(\cdot)^{\top}$ stands for
%vector and matrix transposition, and $\|\mathbf{x}\|$ denotes the $\ell_2$-norm of a vector $\mathbf{x}$. Inequalities for vectors $\mathbf{x} > \mathbf{0}$, and the projection operator $[\mathbf{a}]^+:=\max\{\mathbf{a},\mathbf{0}\}$ are defined entrywise. 
%$\mathbb{E}$ denotes the expectation, while $\langle \cdot, \cdot \rangle$ and $\langle \cdot, \cdot \rangle_{\cal H}$ the vector inner product in Euclidian and  Hilbert space respectively.
\section{Kernel-based learning over graphs}\label{sec:pre}

{\color{black}
Consider a graph $\mathcal{G}(\mathcal{V},\mathcal{E})$ of $N$ nodes, whose topology is captured by a known adjacency matrix $\bbA\in\mathbb{R}^{N\times N}$. Let $a_{nn'}\in\mathbb{R}$ denote the $(n,n')$ entry of $\bbA$, which is nonzero only if an edge is present from node $n'$ to $n$. 
A real-valued function (or signal) on a graph is a mapping $f:  {\cal V} \rightarrow \mathbb{R}$, where ${\cal V}$ is the set of vertices. The value $f(v)=x_v$ represents an attribute of $v \in {\cal V}$, e.g., in the context of brain networks, $x_{v_n}$ could represent the sample of an electroencephalogram (EEG), or functional magnetic resonance imaging (fMRI) measurement at region $n$. In a social network, $x_{v_n}$ could denote the age, political alignment, or annual income of the $n$th person. \textcolor{black}{Suppose that a collection of noisy samples $\{y_m=x_{v_{n_m}}+e_m\}_{m=1}^M$
is available, where $e_m$ models noise, and $M\leq N$ represents the number of measurements. Given $\{y_m\}_{m=1}^M$, and with the graph topology known, the goal is to estimate $f(v)$, and thus reconstruct the graph signal at unobserved vertices. Letting
$ \bby:= [y_1, \ldots, y_M ]^\top$, the observation vector obeys
\begin{align}
\label{eq:measure}
	\bby=\bbPsi\bbx+\bbe
\end{align}
 where $\bbx:=[x_{v_1}, \dots, x_{v_N}]^\top$, $\bbe:=[e_1, \dots, e_M]^\top$, and $\bbPsi\in \{0,1\}^{M\times N}$ is a sampling matrix with binary entries $[\bbPsi]_{m,n_m}=1$ for $m=1, \dots, M$, and $0$, elsewhere.}

Given $\bbPsi$,  $\bby$, and $\bbA$, the goal is to estimate $\bbx$ over the entire network. 
{To tackle the under-determined system \eqref{eq:measure},} consider function $f$ belonging to
 a reproducing kernel Hilbert space (RKHS) defined as~\cite{smola2003kernels,romero2017kernel}
\begin{align}
\label{eq:def:grkhs}
	\mathcal{H}:=\{f:f(v)=\sum_{n=1}^N\alpha_n \kappa(v,v_n), \alpha_n \in \mathbb{R}\}
\end{align}
where $\kappa: \mathcal{V}\times \mathcal{V}\rightarrow \mathbb{R}$ is a pre-selected kernel function. Hereafter, we will let $n_m=m$ for notational convenience, and without loss of generality (wlog). Given  $\bby$, the RKHS-based estimate is formed as
\begin{equation}\label{opt0}
	\hat{f}= \arg \min_{f\in \mathcal{H}}~\frac{1}{M}\sum_{m=1}^M{\cal C}(f(v_{m}),y_m)+\mu\Omega\left(\|f\|_{\mathcal{H}}^2\right)
\end{equation}
where the cost ${\cal C}(\cdot,\cdot)$ can be selected depending on the learning task, e.g., the least-squares (LS) for regression, or the logistic loss for classification; $\|f\|_{\mathcal{H}}^2:=\sum_{n}\sum_{n'} \alpha_n\alpha_{n'}\kappa(v_n,v_{n'})$ is the RKHS norm; $\Omega(\cdot)$ is an increasing function; and, $\mu>0$ is a regularization parameter that {copes with} overfitting.  According to the definition of graph RKHS in \eqref{eq:def:grkhs}, the function estimate can be written as $\hat{f}(v)=\sum_{n=1}^N \alpha_n \kappa(v, v_n):=\bar{\bbalpha}^\top \bar{\mathbf{k}}(v)$, where $\bar{\bm{\alpha}}:=[\alpha_1,\ldots,\alpha_N]^{\top}\!\in\mathbb{R}^N$ collects the basis coefficients, and  $\bar{\mathbf{k}}({v}):=[\kappa(v,v_1),\ldots,\kappa(v,v_N)]^{\top}\!$.
Substituting into the RKHS norm, we find $\|f\|_{\mathcal{H}}^2:=\sum_{n}\sum_{n'} \alpha_n\alpha_{n'}\kappa(v_n,v_{n'})=\bar{\bm{\alpha}}^{\top}\bar{\mathbf{K}}\bar{\bm{\alpha}}$, where the $N\times N$ kernel matrix $\bar{\mathbf{K}}$ has entries $[\bar{\mathbf{K}}]_{n,n'}:=\kappa(v_n,v_{n'})$; thus, the functional problem \eqref{opt0} boils down to 
{\color{black}
\begin{equation}\label{eq:opt1}
	\min_{\bar{\bm{\alpha}}\in\mathbb{R}^N}\!\frac{1}{M}\sum_{m=1}^M{\cal C}(\bar{\bm{\alpha}}^{\top}\bar{\mathbf{k}}(v_{m}),y_m)+\mu \Omega\left(\bar{\bm{\alpha}}^{\top}\bar{\mathbf{K}}\bar{\bm{\alpha}}\right).
\end{equation}
}
According to the representer theorem, the optimal solution of \eqref{opt0} admits the finite-dimensional form given by \cite{smola2003kernels,romero2017kernel}
\begin{equation}\label{eq:sol0}
	\hat{f}(v)=\sum_{m=1}^M\alpha_m \kappa(v,v_{m})
	:=\bm{\alpha}^{\top}{\mathbf{k}}(v).
\end{equation}
where $\bm{\alpha}:=[\alpha_1,\ldots,\alpha_M]^{\top}\!\in\mathbb{R}^M$, and  $\mathbf{k}(v):=[\kappa(v,v_1),\ldots,\kappa(v,v_M)]^{\top}\!$.
This means that the coefficients corresponding to the unobserved nodes are all zeros. This implies that the function over the graph can be estimated by optimizing over the $M\times 1$ vector $\bbalpha$ [cf. \eqref{opt0}]
\begin{equation}\label{eq:opt2}
	 \min_{{\color{black}\bm{\alpha}\in\mathbb{R}^M}}~\frac{1}{M}\sum_{m=1}^M{\cal C}({\bm{\alpha}}^{\top}{\mathbf{k}}(v_m),y_m)+\mu \Omega\left({\bm{\alpha}}^{\top}\bbK{\bm{\alpha}}\right)
\end{equation}
where $\bbK:=\bbPsi^\top\bar{\mathbf{K}}\bbPsi$. For general kernel-based learning tasks, $\bar{\bbK}$ is formed using the nonlinear functions of pairwise correlations $\kappa(v_n, v_{n'})=\bbphi_{n}^\top\bbphi_{n'}$, where $\bbphi_{n}$ denotes the feature vector of node $n$, which can collect, for example, the buying history of users on Amazon, or the trading history of companies in financial networks. However, such information may not be available in practice, due to, e.g., privacy concerns. This has motivated the graph-kernel based approaches in  \cite{romero2017kernel} and \cite{smola2003kernels}, to reconstruct the graph signal when only the network structure is available, and the kernel matrix is selected as a nonlinear function of the graph Laplacian matrix. Specifically, these works mostly consider undirected networks,  $\bbA=\bbA^\top$. 
	
Given the normalized Laplacian matrix $\bbL:=\bbI-\bbD^{-1/2}\bbA\bbD^{-1/2}$, with $\bbD:={\rm diag}(\bbA\mathbf{1})$, and letting $\bbL:=\bbU\bbLambda\bbU^\top$, the family of graphical  kernels is}
\begin{align}
\label{eq:gk}
	\bar{\bbK}:=r^\dagger(\bbL):=\bbU r^\dagger(\bbLambda)\bbU^\top
\end{align}
where $r(.)$ is a non-decreasing scalar function of the eigenvalues, and $^\dagger$ denotes pseudo-inverse. By selecting $r(.)$, different graph properties can be accounted for, including smoothness, band-limitedness, the random walk~\cite{smola2003kernels}, and diffusion~\cite{kondor2002diffusion}.
 
Although graph-kernel based methods are effective in reconstructing signals over graphs, it can be observed from \eqref{eq:gk} that formulating $\bar{\bbK}$ generally requires an eigenvalue decomposition of $\bbL$, which incurs complexity ${\cal O}(N^3)$ that can be prohibitive for large-scale networks. Moreover, even though nodal feature vectors $\{ {\mathbf \phi}_n\}$ are not necessary to form $\bar{\bbK}$, the graph-kernel-based scheme requires knowledge of the topology, meaning $\bbA$, in order to estimate the nodal function of each node. However, in networks with strict privacy requirements, nodes may not be willing to share such information with others. In Facebook, for example, most people do not make their friend list public. In addition, solving \eqref{eq:opt1} assumes that all sampled nodes are available in batch, which may not be true in scenarios where nodes are sampled in a sequential fashion.

In response to these challenges, an online scalable kernel-based method will be developed in the ensuing section to deal with sequentially obtained data samples, over generally dynamic networks. The resultant algorithm only requires encrypted versions of the nodal connectivity patterns of other nodes, and hence it offers privacy. 

\section{Online kernel-based learning over graphs}\label{sec:online}
Instead of resorting to a graph kernel that requires an eigenvalue decomposition of $\bbL$ in \eqref{eq:gk}, the present section advocates treating the \emph{connectivity pattern of each node as its feature vector}, which can be the $n$th column $\bba_n^{(c)}$ and possibly the $n$th row $(\bba_n^{(r)})^\top$ of the adjacency (if $\bbA$ is nonsymmetric). We will henceforth term this the \emph{connectivity pattern} of $v_n$, and denote it as $\bba_n$, for brevity. Given $\bba_n$, we will interpolate unavailable nodal function values $\hat{f}(v_n)$ using a nonparametric approach, that is different and scalable relative to \cite{smola2003kernels} and \cite{romero2017kernel}. The kernel matrix is now
\begin{align}\label{eq:gka}
	[\bar{\bbK}]_{n,n'}=\kappa(v_n,v_{n'})=\kappa(\bba_n,\bba_{n'}).
\end{align}
{Again, with $M$ nodes sampled,} 
the representer theorem asserts that the sought function estimator has the form \cite{wahba1990spline}
\begin{align}
\label{eq:f:graph}
	\hat{f}(v_n)=\hat{f}(\bba_n)=\sum_{m=1}^M\alpha_m \kappa(\bba_m,\bba_n):=\bm{\alpha}^{\top}\mathbf{k}(\mathbf{a}_n)
\end{align}
where $\bbk(\bba_n):=[\kappa(\bba_n,\bba_1) \dots \kappa(\bba_n,\bba_M)]^\top$.
It can be observed from \eqref{eq:f:graph} that  $\hat{f}(v_n)$ involves the adjacency of the entire network, namely $\{\bba_m\}_{m=1}^M$, which leads to potentially growing complexity ${\cal O}(M^3)$ as the number of sampled nodes increases~\cite{wahba1990spline}. 

\subsection{Batch RF-based learning over graphs}
To bypass this growing complexity, we will resort to the so-called  random feature approximation \cite{rahimi2007} in order to reduce the original functional learning task in \eqref{eq:opt1}  to a problem with the number of unknown parameters not growing with $M$. 
We first approximate $\kappa$ in \eqref{eq:sol0} using random features (RFs) \cite{rahimi2007,shen2018aistats} that are obtained from a shift-invariant kernel satisfying $\kappa(\bba_n,\bba_{n'})=\kappa(\bba_n-\bba_{n'})$. For $\kappa(\bba_n-\bba_{n'})$ absolutely integrable,
its Fourier transform $\pi_{\kappa} (\bf v)$ exists and represents the power spectral density, which upon normalizing to ensure $\kappa(\mathbf{0})=1$, can also be 
viewed as a probability density function (pdf); hence,  
\begin{align}
\label{ieq.kx1}
\kappa(\bba_n-\bba_{n'}) &= \int \!\!\pi_{\kappa}(\bbv)e^{j\bbv^\top(\bba_n-\bba_{n'})} d\bbv \nonumber\\
	&:=\mathbb{E}_{\bbv}\big[e^{j\bbv^\top(\bba_n-\bba_{n'})}\big]	
\end{align}
where the last equality is due to the definition of the expected value. Drawing a sufficient number of $D$ independent and identically distributed samples $\{\bbv_i\}_{i=1}^D$ from $\pi_{\kappa}(\bbv)$, the ensemble mean \eqref{ieq.kx1} can be approximated by the sample average  
\begin{equation}\label{eq.ker-quad}
\hat{\kappa}(\bba_n,\bba_{n'})=\bbz_{\bbV}^\top(\bba_n)\bbz_{\bbV}(\bba_{n'})
\end{equation}
where $\bbV:=[\bbv_1, \dots, \bbv_D]^\top \in \mathbb{R}^{D\times N}$, and $\bbz_{\bbV}$ denotes the $2D\times 1$ \emph{real-valued} RF vector  
\begin{align}\label{rep:z}
\bbz_{\bbV}(\bba)&=D^{-\frac{1}{2}}\\
&\times\left[\sin(\bbv_1^\top \bba),\ldots, \sin(\bbv_D^\top \bba), \cos(\bbv_1^\top \bba), \ldots, \cos(\bbv_D^\top \bba)\right]^{\top}\!.\nonumber
\end{align}
Taking expectations in \eqref{eq.ker-quad} and using \eqref{ieq.kx1}, one can verify that
$\mathbb{E}_{\bbv}[\hat{\kappa}(\bba_n,\bba_{n'})]=\kappa(\bba_n,\bba_{n'})$, which means
${\hat \kappa}$ is unbiased. Note that finding $\pi_{\kappa}(\bbv)$ requires an $N$-dimensional Fourier transform of $\kappa$, which in general requires numerical integration. Nevertheless, it has been shown that for a number of popular kernels, $\pi_{\kappa}(\bbv)$ is available in closed form \cite{rahimi2007}. Taking the Gaussian kernel as an example, where $\kappa(\bba_n,\bba_{n'})=\exp\big(\|\bba_n-\bba_{n'}\|_2^2/(2\sigma^2)\big)$, it has a Fourier transform corresponding to the pdf $\mathcal{N}(0,\sigma^{-2}\bbI)$.

 Hence, the function that is optimal in the sense of \eqref{opt0} can be cast to a function approximant over the $2D$-dimensional RF space (cf. \eqref{eq:f:graph} and \eqref{eq.ker-quad})
\begin{align}
\label{eq:rf:fx}
	\hat{f}^{\rm RF}(\bba)=\sum_{m=1}^M \alpha_m \bbz_{\bbV}^\top(\bba_m)\bbz_{\bbV}(\bba):=\bbtheta^\top\bbz_{\bbV}(\bba)
\end{align}
where $\bbtheta^{\top}:=\sum_{m=1}^M \alpha_m \bbz_{\bbV}^{\top}(\bba_m)$.
While $\hat{f}$ in \eqref{eq:sol0} is the superposition of nonlinear functions $\kappa$, its RF approximant $\hat{f}^{\rm RF}$ in \eqref{eq:rf:fx} is a linear function of $\bbz_{\bbV}(\bba_i)$. 
As a result, \eqref{opt0} reduces to
\begin{equation}\label{opt2}
	 \min_{\bbtheta\in \mathbb{R}^{2D}}~\frac{1}{M}\sum_{m=1}^M{\cal C}(\bbtheta^\top\bbz_{\bbV}(\bba_m),y_m)+\mu\Omega\left(\|\bbtheta\|^2\right)
\end{equation}
where  $\|\bbtheta\|^2:=\sum_{t}\sum_{\tau}\alpha_t\alpha_{\tau}\bbz_{\bbV}^\top(\bba_t)\bbz_{\bbV}(\bba_{\tau}):=\|f\|_{\cal H}^2$. A batch solver of \eqref{opt2} has complexity $\mathcal{O}(MD^3)$ that does not grow with $N$. This batch RF-based approach scales linearly with the number of measured nodes $M$, and the number of variables is $2D$, which does not depend on $M$. This allows us to pursue an online implementation as elaborated next.

{\color{black}
\subsection{Online RF-based learning over graphs}
Here, we will further leverage RF-based learning over graphs to enable real-time learning and reconstruction of signals evolving over possibly dynamic networks. A scalable online algorithm will be introduced, which can adaptively handle sequentially sampled nodal features and update the sought function estimates. 

\noindent\textbf{Training sequentially.} In the training phase, we are given a network of $N$ nodes, and the nodal function is sampled in a sequential fashion.
Letting $v_t$ denote the node sampled at the $t$th time slot, and having available $\{\bba_t,y_t\}$ at $v_t$, the online inference task can be written as [cf. \eqref{opt2}]
\begin{align}\label{eq:rf-task}
\hspace{-10mm}&\min_{\bm{\theta}\in\mathbb{R}^{2D}}\, \sum_{\tau=1}^t{\cal L}\left(\bbtheta^\top\bbz_{\bbV}(\bba_{\tau}),y_{\tau}\right)\\
& {\cal L}\big(\bbtheta^\top\bbz_{\bbV}(\bba_t),y_t\big):={\cal C}\big(\bbtheta^\top\bbz_{\bbV}(\bba_t),y_t\big)+\mu \Omega\big(\|\bbtheta\|^2\big).\nonumber
\end{align}
We will solve  \eqref{eq:rf-task} using online gradient descent \cite{hazan2016}. Obtaining $v_t$ per slot $t$, the RF of its connectivity pattern $\bbz_{\bbV}(\bba_t)$ is formed as in \eqref{rep:z}, and $\bbtheta_{t+1}$ is updated `on the fly,' as 
  \begin{align}\label{eq:weit-rf}
  	\bbtheta_{t+1}=\bbtheta_t-\eta_t \nabla{\cal L}(\bbtheta_t^\top\bbz_{\bbV}(\bba_t),y_t)
  \end{align}
where $\{\eta_t\}$ is the sequence of stepsizes that can tune learning rates. 
\textcolor{black}{In this paper, we will adopt $\eta_t=\eta$ for simplicity}. 
Iteration \eqref{eq:weit-rf} provides \emph{a functional update} since $\hat{f}^{\rm RF}_t(\bba)=\bbtheta_t^\top\bbz_{\bbV}(\bba)$.  The per-iteration complexity of \eqref{eq:weit-rf} is $\mathcal{O}(D)$, and $\mathcal{O}(MD)$ for the entire training process, which scales better than~$\mathcal{O}(MD^3)$ that is required for a batch solver of \eqref{opt2}.

\noindent\textbf{Inferring unavailable nodal values.}
After the training phase, the nodal function value over the un-sampled nodes can be readily estimated by [cf. \eqref{eq:rf:fx}]
\begin{align}\label{eq:fx:inter}
	\hat{f}(v_i)=\hat{\bbtheta}^\top \bbz_{\bbV}(\bba_i), ~~\forall i \in {\cal S}^c
\end{align}
\textcolor{black}{where $\hat{\bbtheta}$ is the final estimate after the training phase, i.e., $\hat{\bbtheta}=\bbtheta_{M+1}$, and} ${\cal S}^c$ denotes the index set of the nodes whose signal values have not been sampled in the training phase.

\noindent\textbf{Newly-joining nodes.} When new nodes join the network, batch graph-kernel based approaches must expand $\bar{\bbK}$ in \eqref{eq:gk} by one row and one column, and re-solve \eqref{eq:opt2} in order to form signal estimates for the newly-joining nodes. Hence, each newly joining node will incur complexity $\mathcal{O}(N^3)$. The novel online RF method on the other hand, can simply estimate the signal on the newly coming node via $\hat{f}(v_{\rm new})=\hat{\bbtheta}\bbz_{\bbV}(\bba_{\rm new})$,  where $\bba_{\rm new}\in\mathbb{R}^N$ denotes the connectivity pattern of the new node with the \emph{existing} nodes in the network. This leads to a complexity of $\mathcal{O}(ND)$ per new node. If in addition, $y_{\rm new}$ is available, then the function estimate  can also be efficiently updated  via \eqref{eq:weit-rf} and \eqref{eq:rf:fx} using $\bba_{\rm new}$ and $y_{\rm new}$.

 The steps of our online  RF-based method are summarized in Algorithm~\ref{algo:okl}. A couple of standard learning tasks where Algorithm~\ref{algo:okl} comes handy are now in order.

\emph{Nonlinear regression over graphs.}
Consider first nonlinear regression over graphs, where  the goal is to find a nonlinear function $f\in{\cal H}$, such that $y_n=f(v_n)+e_n=f(\bba_n)+e_n$ given the graph adjacency matrix $\bbA$. The criterion is to minimize the regularized prediction error of $y_n$, typically using the online LS loss ${\cal L}(f(\bba_t),y_t):=[y_t - f(\bba_t)]^2+\mu\|f\|_{\cal H}^2$ in \eqref{eq:rf-task}, whose gradient is (cf. \eqref{eq.klp-weight})
{\color{black}
\begin{align}
& \nabla{\cal L}\left(\bbtheta_{t}^\top\bbz_{\bbV}(\bba_t),y_t\right)=2[\bbtheta_{t}^\top\bbz_{\bbV}(\bba_t)-y_t]\bbz_{\bbV}(\bba_t)+2\mu \bbtheta_{t}.\nonumber
\end{align}
}
In practice, $y_t$ can represent a noisy version of each node's  real-valued attribute, e.g., temperature in a certain city, and the graph can be constructed based on Euclidean distances among cities. 
For a fair comparison with alternatives, only the regression task will be tested in the numerical section of this paper.

\emph{Nonlinear classification over graphs.} We can also handle kernel-based perceptron and kernel-based logistic regression, which aim at learning a nonlinear classifier that best approximates either $y_n$, or, the pdf of $y_n$ conditioned on $\bba_n$. 
With binary labels $\{\pm 1\}$, the perceptron solves \eqref{opt0} with ${\cal L}(f(\bba_t),y_t)=\max(0,1-y_t f(\bba_t))+\mu\|f\|_{\cal H}^2$, which equals zero if $y_t=f(\bba_t)$, and otherwise equals $1$. In this case, the gradient of the presented online RF-based method is (cf. \eqref{eq.klp-weight})
{\color{black}
\begin{align}
	\nabla {\cal L}&\left(\bbtheta_{t}^\top\bbz_{\bbV}(\bba_t),y_t\right)=-2y_t{\cal C}(\bbtheta_{t}^\top\bbz_{\bbV}(\bba_t),y_t)\bbz_{\bbV}(\bba_t)
	+2\mu \bbtheta_{t}.\nonumber
\end{align}
}
Accordingly, given ${\bf x}_t$, logistic regression postulates that ${\rm Pr}(y_t=1|\bbx_t )=1/(1+\exp(f(\bbx_t)))$.
Here the gradient takes the form (cf. \eqref{eq.klp-weight})
\begin{align}
{\color{black}
	\nabla {\cal L}\!\!\left(\bbtheta_{t}^\top\bbz_{\bbV}(\bba_t),y_t\right)\!=\!\frac{2y_t\exp(y_t\bbtheta_{t}^\top\bbz_{\bbV}(\bba_t))}{1+\exp(y_t\bbtheta_{t}^\top\bbz_{\bbV}(\bba_t))}\bbz_p(\bbx_t)+2\mu \bbtheta_{t}.\nonumber
	}
\end{align}
Classification over graphs arises in various scenarios, where $y_n$ may represent categorical attributes such as gender, occupation or, nationality of users in a social network.
 
\noindent\textbf{Remark 1 (Privacy).} Note that the update in \eqref{eq:weit-rf} does not require access to $\bba_t$ directly. Instead, the only information each node needs to reveal is  $\bbz_{\bbV}(\bba_t)$ for each $\bba_t$, which involves $\{\sin(\bba_t^\top \bbv_{j}),~ \cos(\bba_t^\top \bbv_{j})\}_{j=1}^D$. Being noninvertible, these co-sinusoids functions involved in generating the $\mathbf{z}_{\mathbf{V}}(\mathbf{a}_t)$ can be viewed as an encryption of the nodal connectivity pattern, which means that given  $\mathbf{z}_{\mathbf{V}}(\mathbf{a}_t)$, vector $\mathbf{a}_t$ cannot be uniquely deciphered. Hence, Algorithm \ref{algo:okl} preserves privacy. } 

\noindent\textcolor{black}{\textbf{Remark 2 (Directed graphs).} It can be observed from \eqref{eq:gk} that for $\bar{\bbK}$ to be a valid kernel, graph-kernel based methods require $\bbA$, and henceforth $\bbL$ to be symmetric, which implies they can only directly deal with symmetric/undirected graphs. Such a requirement is not necessary for our RF-based method.}

\noindent\textcolor{black}{\textbf{Remark 3 (Dynamic graphs).}
Real-world networks may vary  over time, as edges may disappear or appear. To cope with such changing topologies, the original graph-kernel method needs to recalculate the kernel matrix, and resolve the batch problem whenever one edge changes. In contrast, our online RF-based method can simply re-estimate the nodal values on the two ends of the (dis)appeared edge using \eqref{eq:rf:fx} with their current $\{\bba_n\}$.
}
%\paragraph{Privacy issues.}

%\paragraph{New coming nodes.}

%\textcolor{red}{.............................}

%\noindent\textbf{Remark 3 (Computational complexity).} The per iteration complexity of the novel RF-based algorithm is $\mathcal{O}(DN)$, compared with the batch solver of complexity $\mathcal{O}(N^3)$~ \cite{romero2017kernel}. 
%Moreover, when a new node join the network, the Gradraker is capable of providing real-time estimate at the complexity of $\mathcal{D}$, while the conventional batch algorithm, each newly joining node will cause a complexity of $\mathcal{O}(N^3)$ in order to obtain the estimate.

Evidently, the performance of Algorithm \ref{algo:okl} depends on $\kappa$ that is so far considered known. As the ``best'' performing $\kappa$ is generally unknown and application dependent, 
it is prudent to adaptively select kernels by superimposing multiple kernels from a prescribed dictionary, as we elaborate next.

%%%%%%%%%%%%%%%%%%%%%%%%%%
\begin{algorithm}[t] %[h]
	\caption{Online kernel based learning over graphs}\label{algo:okl}
	\begin{algorithmic}[1]
		\State\textbf{Input:} step size $\eta>0$, and number of RFs $D$.
		
		%%%%%%%%%%%%%%%%%%%%%
		
		\State\textbf{Initialization:}~$\bbtheta_{1}=\mathbf{0}$.
		\State\textbf{Training:}
		\For {$t = 1, 2,\ldots, M$}

		%%%%%%%%%%%%%%%%%%%%%%%%%
			
		\State Obtain the adjacency vector $\bba_t$ of  sampled node $v_t$ . 
		\State Construct $\bbz_{p}(\bba_t)$ via \eqref{rep:z} using $\kappa$. 
		%\State Predict $\hat{f}_t^{\rm RF}(v_t)=  \hat{f}_{t}^{\rm RF}(\bba_t)$ in \eqref{eq:rf:fx}.
		%with $\hat{f}_{p,t}^{\rm RF}(\bba_t)$ in \eqref{eq.klp-output}.  
	%	\State Observe loss function ${\cal L}_t$, incur ${\cal L}_t(\hat{f}_t^{\rm RF}(\bba_t))$.
		%\Return $\hat{\bbA}$, $\widehat{\bbB}$
		\State Update $\bbtheta_{t+1}$ via \eqref{eq:weit-rf}.
		\EndFor
		\State \textbf{Inference: }\\
		\hspace{6mm}Construct random feature vector $\bbz_{\bbV}(\bba_j)$ via \eqref{rep:z}\\
		\hspace{6mm}Infer
		 $\hat{f}(v_j)=\bbtheta_{M+1}^\top\bbz_{\bbV}(v_j),~~j\in \Omega.$
		 \State \textbf{Accounting for newly-coming node}\\
		 \hspace{6mm}Construct random feature vector $\bbz_{\bbV}(\bba_{\rm new})$ via \eqref{rep:z}\\
		\hspace{6mm}Estimate
		 $\hat{f}(v_{\rm new})=\bbtheta_{M+1}^\top\bbz_{\bbV}(v_{\rm new}).$\\
		 \hspace{6mm}If $y_{\rm new}$ available, Update $\bbtheta$ via \eqref{eq:weit-rf}.
	\end{algorithmic}
\end{algorithm}
%%%%%%%%%%%%%%%%%%%%%%%%%%%%%%%%%%%%%%%%%%%%%%%%%%%%%%%

\section{Online Graph-adaptive MKL}\label{sec:gradraker}

In the present section, we develop an online \textbf{gr}aph-\textbf{ad}aptive learning approach that relies on \textbf{ra}ndom features, and leverages multi-\textbf{ker}nel approximation  to estimate the desired $f$ based on sequentially obtained nodal samples over the graph. The proposed method is henceforth abbreviated as \textbf{Gradraker}. 

The choice of $\kappa$ is critical for the performance of single kernel based learning over graphs, since different kernels capture different properties of the graph, and thus lead to function estimates of variable accuracy \cite{romero2017kernel}. To deal with this, combinations of kernels from a preselected dictionary $\{\kappa_p\}_{p=1}^P$ can be employed in \eqref{opt0}; see also \cite{romero2017kernel,shen2018aistats}. Each combination belongs to the convex hull $\bar{{\cal K}}:=\{\bar{\kappa}=\sum_{p=1}^P \bar{\alpha}_p \kappa_p,\, \bar{\alpha}_p\geq 0,\,\sum_{p=1}^P\bar{\alpha}_p=1\}$. With $\bar{\cal H}$ denoting the RKHS induced by $\bar{\kappa}\in \bar{{\cal K}}$, one then solves \eqref{opt0} with ${\cal H}$ replaced by 
$\bar{\cal H}:={\cal H}_1\bigoplus\cdots\bigoplus{\cal H}_P$,
where $\{{\cal H}_p\}_{p=1}^P$ represent the RKHSs corresponding to $\{\kappa_p\}_{p=1}^P$~\cite{cortes2009}.

The candidate function ${\bar f} \in \bar{\cal H}$ is expressible in a separable form as $\bar{f}(\bba): =\sum_{p=1}^P {\bar f}_p(\bba)$, where ${\bar f}_p(\bba)$ belongs to $\mathcal{H}_p$, for $p\in{\cal P}:=\{1, \ldots, P\}$. To add flexibility per kernel in our ensuing online MKL scheme, we let wlog $\{{\bar f}_p = {w}_p f_p\}_{p=1}^P$, and seek functions of the form 
\begin{align}\label{eq:fp}
f(v)=f(\bba): =\sum_{p=1}^P \bar{w}_p f_p(\bba)\in\bar{\cal H}
\end{align}
where $f:={\bar f}/\sum_{p=1}^P w_p$, and the normalized weights $\{\bar{w}_p:=w_p/\sum_{p=1}^P w_p\}_{p=1}^P$ satisfy $\bar{w}_p\geq 0$, and $\sum_{p=1}^P\bar{w}_p=1$. 
Exploiting separability jointly with the 
RF-based function approximation per kernel, the  MKL task can be reformulated, after letting ${\cal L}_t(\hat{f}_p^{\rm RF}(\bba_t)):={\cal L}\big(\bbtheta^\top\bbz_{\bbV_p}(\bba_t),y_t\big)$ in \eqref{eq:rf-task}, as
\begin{subequations}\label{eq:raker-task}
\begin{align}
\!&\min_{\{\bar{w}_p\}, \{\hat{f}_p^{\rm RF}\}}\, \sum_{t=1}^T\sum_{p=1}^P \bar{w}_p\, {\cal L}_t\left(\hat{f}_p^{\rm RF}(\bba_t)\right)\label{eq:raker-taska}\\
{\rm s.~to}&~~\sum_{p=1}^P\bar{w}_p=1,~\bar{w}_p\geq 0,~p\in{\cal P},\\
&~~\hat{f}_p^{\rm RF}\!\in\!\left\{\hat{f}_p(\bba_t)\!=\!\bbtheta^{\top}\bbz_{\bbV_p}(\bba_t)\right\},~p\in{\cal P}\label{eq:raker-taskc}
\end{align}
\end{subequations}
which can be solved efficiently `on-the-fly.' Relative to \eqref{opt2}, we replaced $M$ by $T$ to introduce the notion of time, and stress the fact that the nodes are sampled sequentially. 

Given the connectivity pattern  $\bba_t$ of the $t$th sampled node $v_t$, an RF vector $\bbz_p(\bba_t)$ is generated per $p$ from the pdf $\pi_{\kappa_p}(\bbv)$  via \eqref{rep:z}, where $\bbz_p(\bba_t):=\bbz_{\bbV_p}(\bba_t)$ for notational brevity. Hence, per kernel  $\kappa_p$ and node sample $t$, we have [cf. \eqref{eq:rf:fx}]
\begin{align}
  	\hat{f}_{p,t}^{\rm RF}(\bba_t)=\bbtheta_{p,t}^\top \bbz_p(\bba_t)
\end{align}
and as in \eqref{eq:weit-rf}, $\bbtheta_{p,t}$ is updated via
  \begin{align}\label{eq.klp-weight}
  	\bbtheta_{p,t+1}=\bbtheta_{p,t}-\eta \nabla{\cal L}(\bbtheta_{p,t}^\top\bbz_p(\bba_t),y_t)
  \end{align}
with $\eta\in(0,1)$ chosen constant to effect the adaptation.
As far as  $\bar{w}_{p,t}$ is concerned, since it resides on the probability simplex, a multiplicative update is well motivated as discussed also in, e.g.,~\cite{hazan2016, shen2018aistats}. For the un-normalized weights, this update is available in closed form as \cite{shen2018aistats}
\begin{align}\label{eq.mkl-weight}
w_{p,t+1}=w_{p,t}\exp\left(-\eta{\cal L}_t\left(\hat{f}_{p,t}^{\rm RF}(\bba_t)\right)\right).
\end{align}
 Having found $\{w_{p,t}\}$ as in \eqref{eq.mkl-weight}, the normalized weights in \eqref{eq:fp} are obtained as $\bar{w}_{p,t}:=w_{p,t}/\sum_{p=1}^P w_{p,t}$. 
Note from \eqref{eq.mkl-weight} that when $\hat{f}_{p,t}^{\rm RF}$ has a larger loss relative to other $\hat{f}_{p',t}^{\rm RF}$ with $p' \neq p$ for the $t$th sampled node, the corresponding $w_{p,t+1}$ decreases more than the other weights. In other words, a more accurate approximant tends to play a more important role in predicting the ensuing sampled node.
In summary, our Gradraker for online graph MKL is listed as Algorithm \ref{algo:omkl:rf}. 

{\color{black}
\noindent\textbf{Remark 4 (Comparison with batch MKL).} A batch MKL based approach for signal reconstruction over graphs was developed in \cite{romero2017kernel}. It entails an iterative algorithm whose complexity grows with $N$ in order to jointly estimate the nodal function, and to adaptively select the kernel function. When new nodes join the network, \cite{romero2017kernel} re-calculates the graphical kernels and re-solves the overall batch problem, which does not scale with the network size. In addition, \cite{romero2017kernel} is not privacy preserving in the sense that in order to estimate the function at any node, one needs to have access to the connectivity pattern of the entire network.

\noindent\textbf{Remark 5 (Comparison with $k$-NN).} An intuitive yet efficient way to predict function values of a newly joining node is to simply combine the values of its $k$ nearest neighbors ($k$-NN) \cite{altman1992introduction,chen2009fast}. Efficient as it is, $k$-NN faces  several challenges: a) At least one of the neighbors must be labeled, which does not always hold in practice, and is not required by the Gradraker; and b) $k$-NN can only account for local information, while the Gradraker takes also into account the global information of the graph.}

%\noindent\textbf{Remark 6. (Generalzations)}

%%%%%%%%%%%%%%%%%%%%%%%%%%
\begin{algorithm}[t] %[h]
	\caption{Gradraker algorithm}\label{algo:omkl:rf}
	\begin{algorithmic}[1]
		\State\textbf{Input:}~Kernels $\kappa_p, ~p=1,\ldots, P$, step size $\eta>0$, and number of RFs $D$.
		
		%%%%%%%%%%%%%%%%%%%%%
		
		\State\textbf{Initialization:}~$\bbtheta_{p,1}=\mathbf{0}$.
		\State\textbf{Training:}
		\For {$t = 1, 2,\ldots, T$}

		%%%%%%%%%%%%%%%%%%%%%%%%%
		\State Obtain the adjacency vector $\bba_t$ of node $v_t$ . 
		\State Construct $\bbz_{p}(\bba_t)$ via \eqref{rep:z} using $\kappa_p$ for $p=1,\dots, P$. 
		\State Predict $\hat{f}_t^{\rm RF}(\bba_t)=\sum_{p=1}^P \bar{w}_{p,t} \hat{f}_{p,t}^{\rm RF}(\bba_t)$ 
		%with $\hat{f}_{p,t}^{\rm RF}(\bba_t)$ in \eqref{eq.klp-output}.  
		\State Observe loss function ${\cal L}_t$, incur ${\cal L}_t(\hat{f}_t^{\rm RF}(\bba_t))$.
		\hspace{1.cm} \For {$p=1, \ldots, P$}
		\State Obtain loss ${\cal L}(\bbtheta_{p,t}^\top\bbz_p(\bba_t),y_t)$ or ${\cal L}_t(\hat{f}_{p,t}^{\rm RF}(\bba_t))$.
		\State Update $\bbtheta_{p,t+1}$ and $w_{p,t+1}$ via \eqref{eq.klp-weight} and \eqref{eq.mkl-weight}.
		\EndFor
		%\Return $\hat{\bbA}$, $\widehat{\bbB}$
		\EndFor
		\State \textbf{Inference: }\\
		\hspace{6mm}Construct RF vector $\{\bbz_{p}(\bba_j)\}$  using $\{\kappa_p\}$.\\
		\hspace{6mm}Infer
		 $\hat{f}(v_j)=\sum_{p=1}^P \bar{w}_{p,T+1} \bbtheta_{p,T+1}^\top\bbz_{p}(v_j).$
		 \State \textbf{Accounting for newly-coming node}\\
		 \hspace{6mm}Construct RF vector $\{\bbz_{p}(\bba_{\rm new})\}$  using $\{\kappa_p\}$.\\
		\hspace{6mm}Estimate
		 $\hat{f}(v_{\rm new})=\sum_{p=1}^P\bar{w}_{p,T+1} \bbtheta_{p,T+1}^\top\bbz_{p}(v_{\rm new}).$\\
		 \hspace{6mm}If $y_{\rm new}$ available update $\{\bbtheta_p, w_{p}\}$ via \eqref{eq.klp-weight} and \eqref{eq.mkl-weight}.
	\end{algorithmic}
\end{algorithm}

\subsection{Generalizations}
So far, it is assumed that each node $n$ only has available its own connectivity feature vector $\bba_n$. This allows Gradraker to be applied even when limited information is available about the nodes, which many existing algorithms that rely on nodal features cannot directly cope with.

If additional feature vectors $\{ \boldsymbol{\phi}_{i,n}\}_{i=1}^I$ are actually available per node $n$ other than its own $\bba_n$, it is often not known a priori which set of features  is the most informative for estimating the signal of interest on the graph. To this end, the novel Gradraker can be adapted by treating the functions learned from different sets of features as \emph{an ensemble of learners}, and combine them in a similar fashion as in \eqref{eq:fp}, that is, 
\begin{align}
	f(v_n)=\sum_{i=1}^I \beta_i f_i(\bbphi_{i,n})\label{eq:f:fad}
\end{align}
Applications to several practical scenarios are discussed in the following.
  
\noindent\textbf{Semi-private networks.}
%So far it is assumed that nodes in the networks have high privacy concerns, and are not willing to reveal their connectivity pattern to anyone in the network. ]
In practice, a node may tolerate sharing its links to its neighbors, e.g., users of Facebook may share their friends-list with friends. In this scenario, each node not only knows its own neighbors, but also has access to who are  its neighbors' neighbors, i.e., two-hop neighbors. Specifically, node $n$ has access to $\bba_n$, as well as to the $n$th column of \textcolor{black}{$\bbA^{(2)}:=\bbA\bbA$} \cite{kolaczyk2009statistical}, and a learner $f_2(\bbphi_{2,n})$ can henceforth be introduced and combined in \eqref{eq:f:fad}.  Moreover, when nodes are less strict about privacy, e.g., when a node is willing to share its multi-hop neighbors, more learners can be introduced and combined `on the fly' by selecting $\bbphi_{i,n}$ as the $n$th column of $\bbA^{(i)}$ in \eqref{eq:f:fad}.

\noindent\textbf{Multilayer networks.}
Despite their popularity, ordinary networks are often inadequate to describe increasingly complex systems. For instance, modeling interactions between two individuals using a single edge can be a gross simplification of reality. Generalizing their \emph{single-layer} counterparts, \emph{multilayer networks} allow nodes to belong to $N_g$ groups, called layers~\cite{kivela2014multilayer,traganitis2017}. These layers could represent different attributes or characteristics of a complex system, such as temporal snapshots of the same network, or different types of groups in social networks (family, soccer club, or work related). Furthermore, multilayer networks are able to model systems that typically cannot be represented by traditional graphs, such as heterogeneous information networks~\cite{zhou2007co,sun2013mining}. 
%traganitis2017
%
To this end, Gradraker  can readily incorporate the information collected from heterogenous sources, e.g., connectivity patterns $\{\bbA_i\}_{i=1}^{N_g}$ from different layers, by adopting a kernel based learner $f_i(\bba_{i,n})$ on the $i$th layer and combining them as in \eqref{eq:f:fad}.

\noindent\textbf{Nodal features available.}
In certain cases,  nodes may have nodal features \cite{kolaczyk2009statistical} in addition to their $\{\bba_n\}$. For example,  in social networks, other than the users' connectivity patterns, we may also have access to their shopping history on Amazon. In financial networks, in addition to the knowledge of trade relationships with other companies, there may be additional information available per company, e.g., the number of employees, category of products  the company sales, or the annual profit.
 Gradraker can also incorporate this information  by introducing additional learners  based on the nodal feature vectors, and combine them as in \eqref{eq:f:fad}. 

%Furthermore, the Fadraker is also privacy-preserving. Meaning the $n$th node only need to share encrypted version of $\{\bbz_{\bbV}(\bbphi_{i,n})\}_{i=1}^I$ with other nodes in the network.
%Moreover, in case when a new node join the group, the nodal attribute can be estimated via simply combining the weighted estimates of the kernel based learner depends on which set(s) of features is (are) available. 

\section{Performance analysis}

To analyze the performance of the novel Gradraker algorithm, we assume that the following are satisfied.

\vspace{0.1cm}
\noindent\textbf{(as1)}
\emph{For all sampled nodes $\{v_t\}_{t=1}^T$, the loss function ${\cal L}(\bbtheta^\top\bbz_{\bbV}(\bba_t),y_t)$ in \eqref{eq:rf-task} is convex w.r.t. $\bbtheta$.}

\vspace{0.1cm}
\noindent\textbf{(as2)}
\emph{For $\bm{\theta}$ belonging to a bounded set ${\bbTheta}$ with $\|\bbtheta\|\leq C_{\theta}$, the loss is bounded; that is, ${\cal L}(\bbtheta^\top\bbz_{\bbV}(\bba_t),y_t)\in[-1,1]$, and has bounded gradient, meaning, $\|\nabla {\cal L}(\bbtheta^\top\bbz_{\bbV}(\bba_t),y_t)\|\leq L$.}

\vspace{0.1cm}
\noindent\textbf{(as3)}
\emph{The kernels $\{\kappa_p\}_{p=1}^P$ are shift-invariant, standardized, and bounded, that is, $\kappa_p(\mathbf{a}_n,\mathbf{a}_{n'})\!\leq\! 1,\,\forall \mathbf{a}_n,\mathbf{a}_{n'}$; and w.l.o.g. they also have bounded entries, meaning $\|\mathbf{a}_n\|\leq 1, \forall n$.}  
\vspace{0.1cm}

Convexity of the loss under (as1) is satisfied by the popular loss functions including the square loss and the logistic loss.
As far as (as2), it ensures that the losses, and their gradients are bounded, meaning they are $L$-Lipschitz continuous. 
While boundedness of the losses commonly holds since $\|\bbtheta\|$ is bounded, Lipschitz continuity is also not restrictive. Considering kernel-based regression as an example, the gradient is $(\bm{\theta}^{\top} \mathbf{z}_{\bbV}(\mathbf{x}_t)-y_t) \mathbf{z}_{\bbV}(\mathbf{x}_t)+\lambda\bbtheta$. Since the loss is bounded, e.g., $\|\bm{\theta}^{\top} \mathbf{z}_{\bbV}(\mathbf{x}_t)-y_t\| \leq 1$, and the RF vector in \eqref{rep:z} can be bounded as $\|\mathbf{z}_{\bbV}(\mathbf{x}_t)\|\leq 1$, the constant is $L:= 1+\lambda C_{\theta}$ using the Cauchy-Schwartz inequality. 
Kernels satisfying the conditions in (as3) include Gaussian, Laplacian, and Cauchy \cite{rahimi2007}.   
In general, (as1)-(as3) are standard in online convex optimization (OCO) \cite{shalev2011,hazan2016}, and in kernel-based learning \cite{micchelli2005,rahimi2007,lu2016large}.
  
% For common applications like regression and classification, e.g. LS or logistic loss, as1) and as2) are satisfied. While as3) is introduced for notational simplicity, in fact, the conclusions in the present section hold with some difference in constants as long as ${\cal L}(\bbtheta^\top\bbz_{\bbV}(\bba_t),y_t), $$\{\kappa_p\}$ and $\{\bba_n\}$ are bounded.
In order to quantify the performance of Gradraker, we resort to the  static regret metric, which  quantifies the difference
between the aggregate loss of an OCO algorithm, and that of the best
fixed function approximant in hindsight, see also e.g.,~\cite{shalev2011,hazan2016}. Specifically, for a  sequence $\{\hat{f}_t\}$ obtained by an online algorithm ${\cal A}$, its static regret is
\begin{align}\label{eq.sta-reg}
    {\rm Reg}_{\cal A}^{\rm s}(T):=\sum_{t=1}^T {\cal L}_t(\hat{f}_t(\bba_t))-\sum_{t=1}^T{\cal L}_t(f^*(\bba_t))
\end{align}
where $\hat{f}_t^{\rm RF}$ will henceforth be replaced by $\hat{f}_t$  for notational brevity; and, $f^*(\cdot)$ is defined as the batch solution
\begin{align}\label{eq.slot-opt}
 f^*(\cdot)  & \in\arg\min_{\{f_p^*,\,p\in{\cal P}\}}\,\sum_{t=1}^T {\cal L}_t(f_p^*(\bba_t))\nonumber\\& ~~~{\rm with}~~~f_p^*(\cdot)\in\arg\min_{f\in{\cal F}_p} \,\sum_{t=1}^T {\cal L}_t(f(\bba_t))
\end{align}
where ${\cal F}_p:={\cal H}_p$, with ${\cal H}_p$ representing the RKHS induced by $\kappa_p$. 
We establish the  regret of our Gradraker approach in the following lemma.

 \begin{lemma}
\label{lemma4}
	Under (as1), (as2), and with $\hat{f}_p^*$
	{\color{black} defined as
	$
		\hat{f}_p^*(\cdot)\in\arg\min_{f\in
		\hat{\cal F}_p} \,\sum_{t=1}^T {\cal L}_t(f(\bba_t))
	$,
	}
	 with $\hat{\cal F}_p:=\{\hat{f}_p|\hat{f}_p(\bba)=\bbtheta^{\top}\mathbf{z}_p(\bba),\,\forall \bbtheta\in\mathbb{R}^{2D}\}$, for any $p$, the sequences $\{\hat{f}_{p,t}\}$ and $\{\bar{w}_{p,t}\}$ generated by Gradraker satisfy the following bound 
\begin{align}
	\label{eq.mkl.sreg}
&\sum_{t=1}^T{\cal L}_t\bigg(\sum_{p=1}^P \bar{w}_{p,t} \hat{f}_{p,t}(\bba_t)\bigg)-\sum_{t=1}^T{\cal L}_t(\hat{f}_p^*(\bba_t))\nonumber\\
\leq & \frac{\ln P}{\eta}+\frac{\|\bbtheta_p^*\|^2}{2\eta}+\frac{\eta L^2T}{2}+\eta T
	\end{align}
where $\bbtheta_p^*$ is associated with the best RF function approximant $\hat{f}_p^*(\bba)=\left(\bbtheta_p^*\right)^{\top}\mathbf{z}_p(\bba)$.
\end{lemma}
\begin{proof}
	See Appendix \ref{app.pf.lemma4}
\end{proof}

In addition to bounding the regret in the RF space, the next theorem compares the Gradraker loss 
relative to that of the best functional estimator in the original RKHS.
\begin{theorem}\label{theorem0}
Under (as1)-(as3), and with $f^*$ defined as in \eqref{eq.slot-opt}, for a fixed $\epsilon>0$, the following bound holds with probability at least $1-2^8\big(\frac{\sigma_p}{\epsilon}\big)^2 \exp \big(\frac{-D\epsilon^2}{4N+8}\big)$
\begin{align}\label{eq.sreg.f}
	&\sum_{t=1}^T{\cal L}_t\left(\sum_{p=1}^P \bar{w}_{p,t} \hat{f}_{p,t}(\bba_t)\right)-\!\!\!\sum_{t=1}^T{\cal L}_t\left(f^*(\bba_t)\right)\nonumber\\
	\leq &\frac{\ln P}{\eta}+\frac{(1+\epsilon)C^2}{2\eta}\!+\!\frac{\eta L^2T}{2}+\eta T\!+\!\epsilon LTC
\end{align}
where $C$ is a constant, while $\sigma_p^2:=\mathbb{E}_{\pi_{\kappa_p}}[\|\bbv\|^2]$ is the second-order moment of the RF vector norm. Setting $\eta=\epsilon={\cal O}(1/\sqrt{T})$ in \eqref{eq.sreg.f}, the static regret in \eqref{eq.sta-reg} leads to
\begin{align}
\label{eq:sreg:11}
	 {\rm Reg}_{\rm Gradraker}^{\rm s}(T)= {\cal O}(\sqrt{T}).
\end{align}
%where the benchmark in \eqref{eq.slot-opt} is from the RKHS ${\cal F}:=\bigcup_{p\in{\cal P}}{\cal H}_p$.
\end{theorem}
\begin{proof}
	See Appendix \ref{app.pf.theorem0}
\end{proof}

\begin{figure*}[t]
\centering
	\begin{minipage}[b]{.49\textwidth}
		\centering
		\includegraphics[width=9cm]{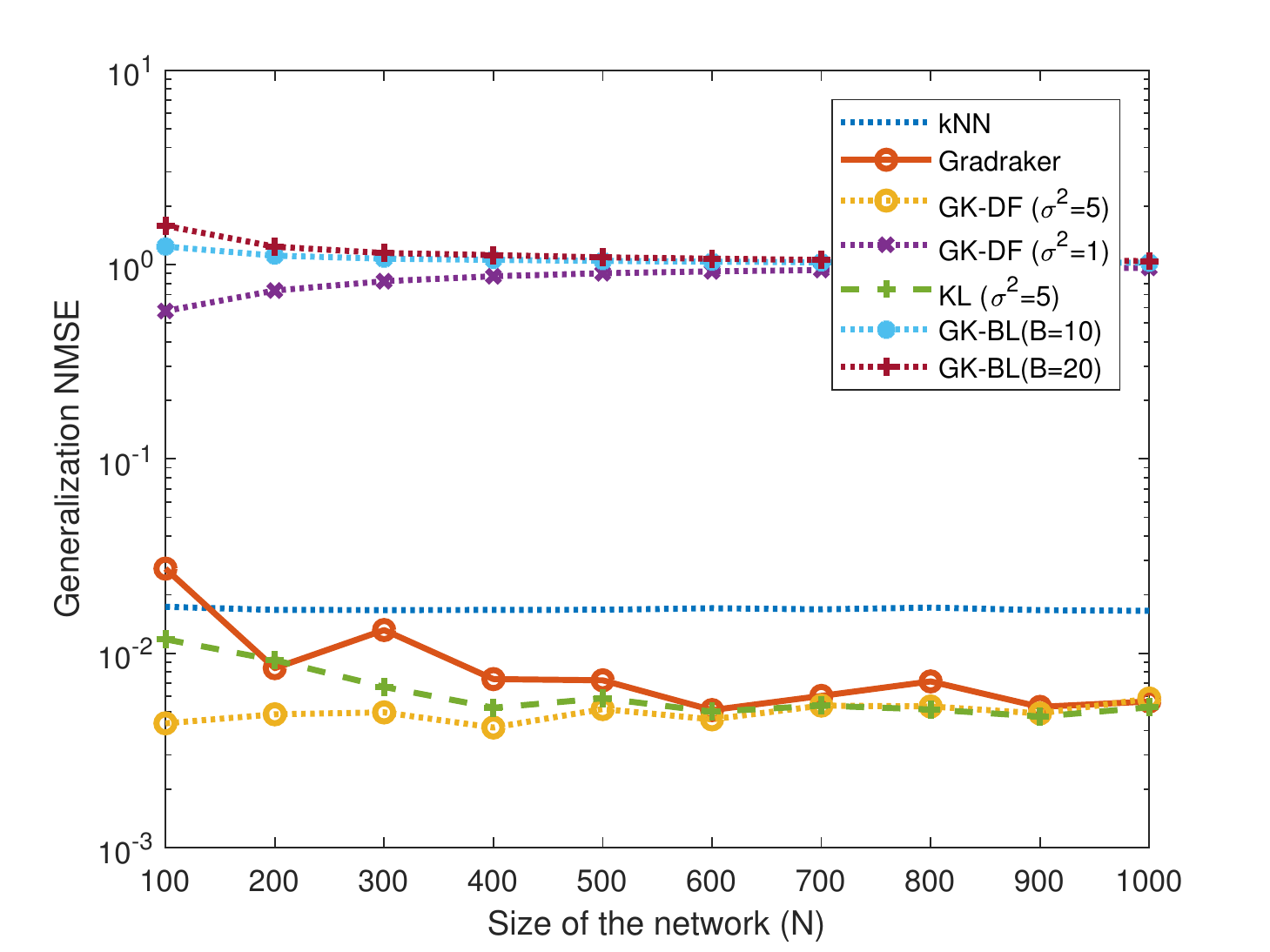}
		\centerline{(a) Generalization NMSE}
	\end{minipage}
	\begin{minipage}[b]{.49\textwidth}
		\centering
		\includegraphics[width=9cm]{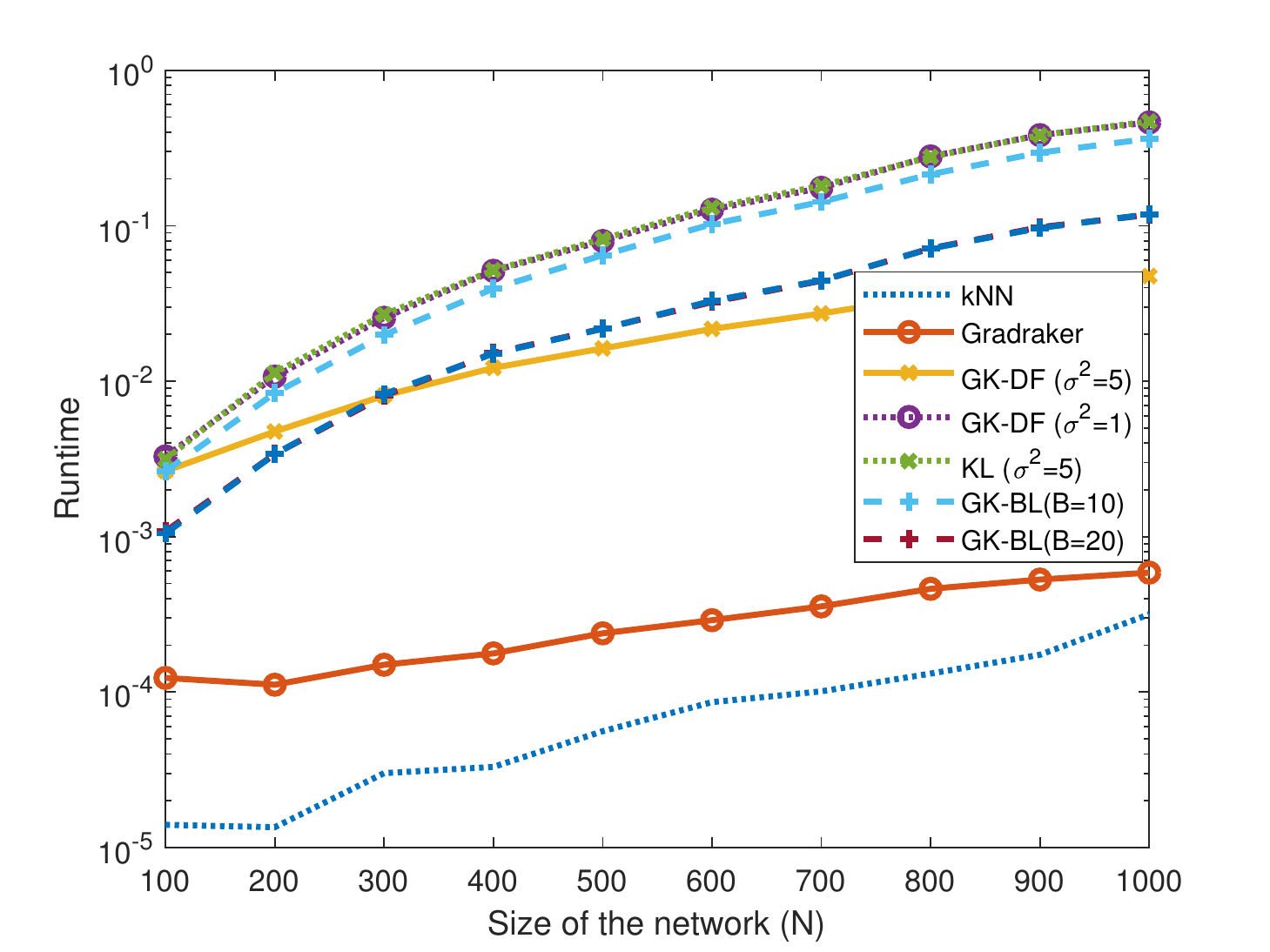}
		\centerline{(b) Testing runtime }
	\end{minipage}
	\caption{{Inference performance versus number of nodes for synthetic dataset generated from graph diffusion kernel  } }
	\label{fig1}
\end{figure*}

Observe that the probability of \eqref{eq.sreg.f} to hold grows as $D$ increases, and one can always find a $D$ to ensure a positive probability for a given $\epsilon$.  Theorem \ref{theorem0} establishes that with a proper choice of parameters, the Gradraker achieves sub-linear regret relative to the best static function approximant in \eqref{eq.slot-opt}, which means the novel Gradraker algorithm is capable of  capturing the nonlinear relationship among nodal functions accurately, as long as enough nodes are sampled sequentially. 

In addition, it is worth noting that Theorem \ref{theorem0} holds true  regardless of the sampling order of the nodes $\{v_1, \dots, v_T\}$. However, optimizing over the sampling pattern is possible, and constitutes one of  our future research directions.

\begin{figure*}[t]
\centering
	\begin{minipage}[b]{.49\textwidth}
		\centering
		\includegraphics[width=9cm]{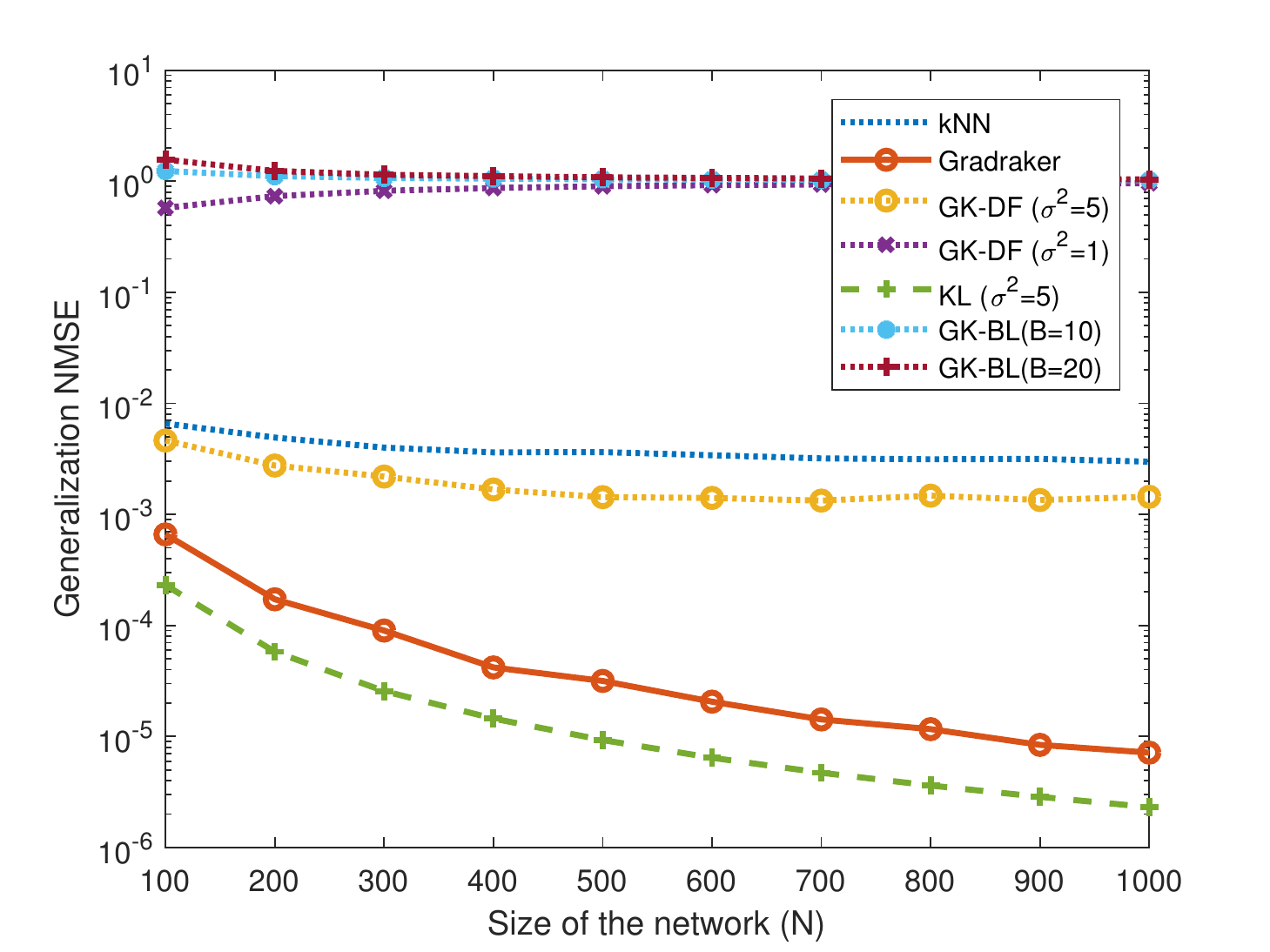}
		\centerline{(a) Generalization NMSE}
	\end{minipage}
	\begin{minipage}[b]{.49\textwidth}
		\centering
		\includegraphics[width=9cm]{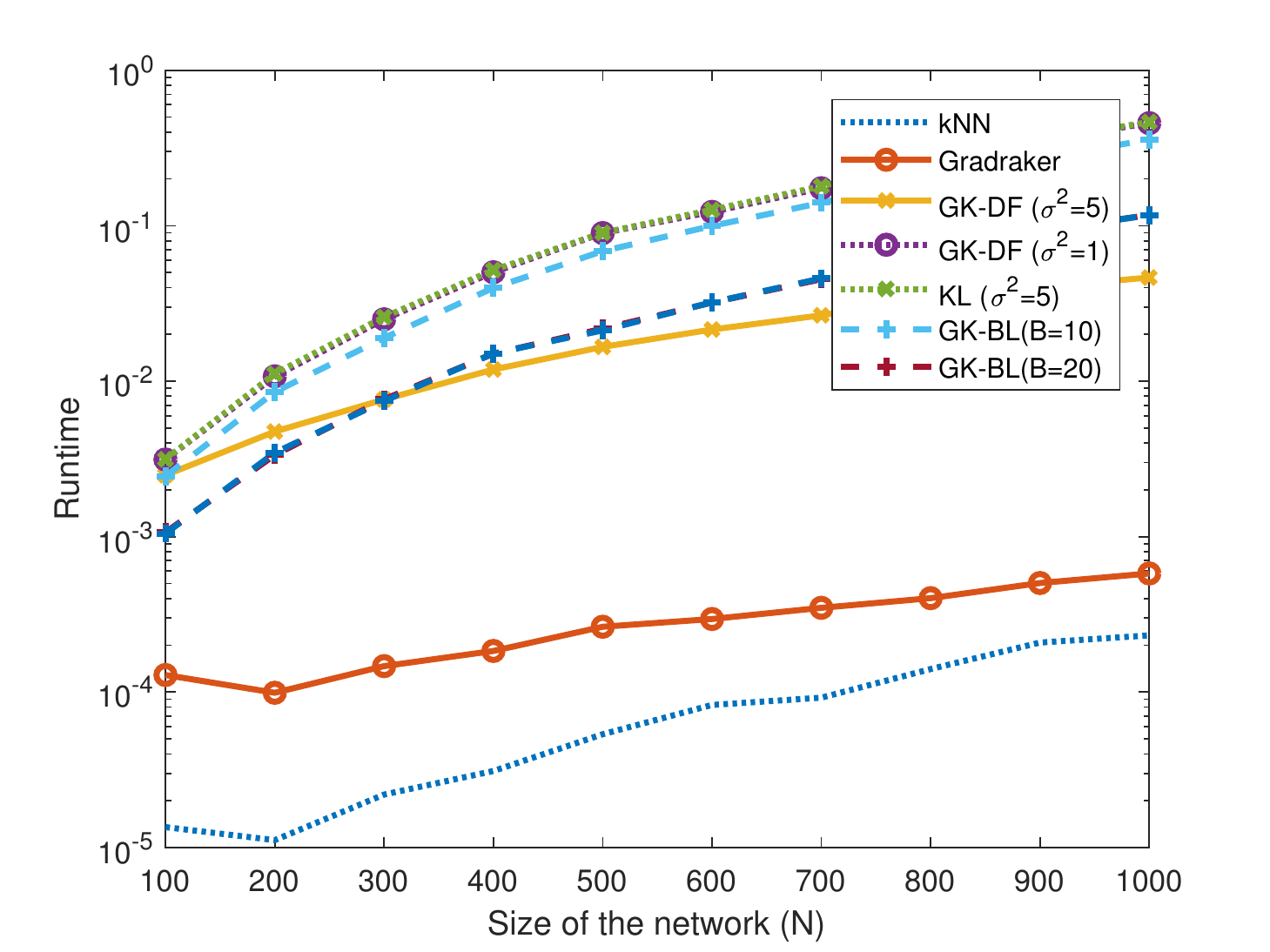}
		\centerline{(b) Runtime}
	\end{minipage}
	\caption{Inference performance versus number of nodes for synthetic dataset generated from Gaussian kernel  }
	\label{fig2}
\end{figure*}

\section{Numerical tests}
\label{sec:test}

In this section, Gradraker is tested on both synthetic and real datasets to corroborate its effectiveness. The tests will mainly focus on regression tasks for a fair comparison with existing alternatives.

\begin{figure*}[t]\label{fig:temp}
\centering
	\begin{minipage}[b]{.49\textwidth}
		\centering
		\includegraphics[width=9cm]{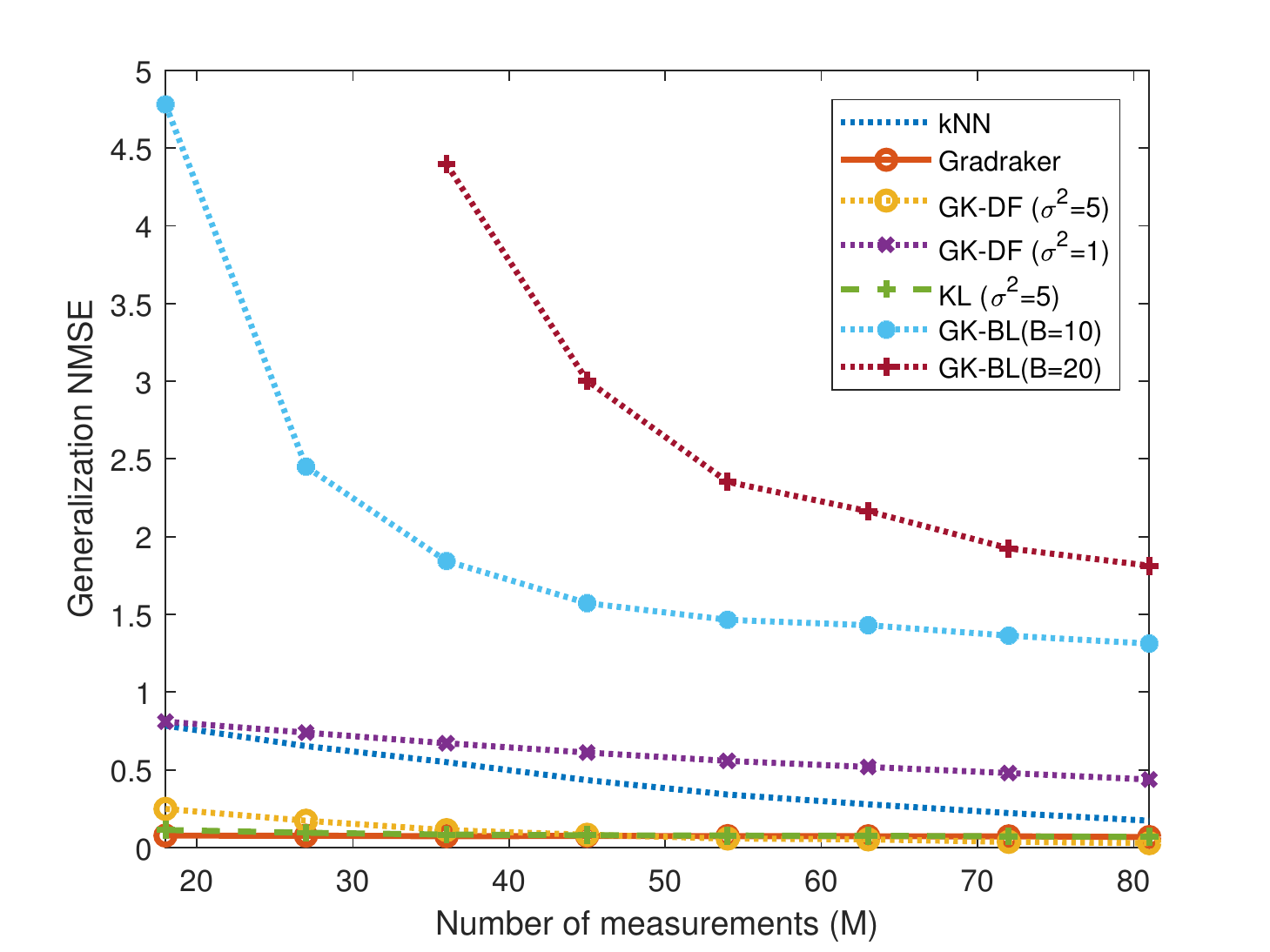}
		\centerline{(a) Generalization NMSE}
	\end{minipage}
	\begin{minipage}[b]{.49\textwidth}
		\centering
		\includegraphics[width=9cm]{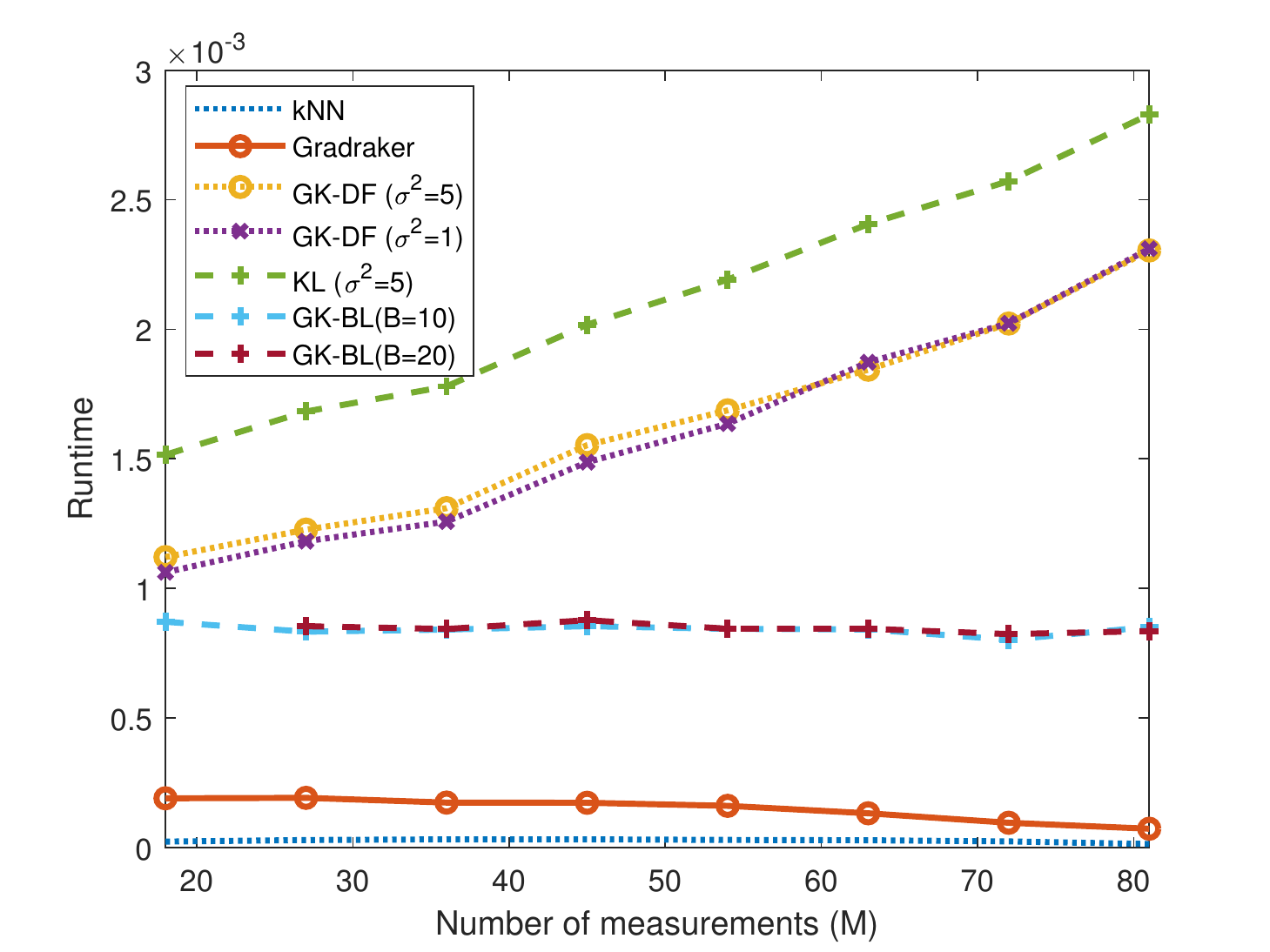}
		\centerline{(b) Runtime}
	\end{minipage}
		\caption{Inference performance versus number of sampled nodes in temperature dataset  }
	\label{fig:temp}
\end{figure*}

\subsection{Synthetic data test}

\noindent\textbf{Data generation.}
An Erd{\"o}s-R{\'e}nyi graph \cite{erdos1959random} with binary adjacency matrix  $\bbA_0\in \bbR^{N\times N}$ was generated with probability of edge presence $\pi=0.2$, and its adjacency was  symmetrized as ${\bbA}=\bbA_0+\bbA_0^\top$. This symmetrization is not required by Gradraker, but it is necessary for alternative graph kernel based methods. A function over this graph was then generated with
 each entry of the coefficient vector $\bbalpha\in\mathbb{R}^{N}$ drawn uniformly from $[0.5,1]$, and each entry of the noise $\bbe$ drawn from $\mathcal{N}(0, 0.01\bbI)$. In each experiment, the sampling matrix  $\bbPsi$ is randomly generated so that $M=0.05 N$ of the nodes are randomly sampled, and the remaining $N-M$ nodes are treated as newly-joining nodes, whose function values and connectivity patterns are both unknown at the training phase, and whose nodal function values are estimated based on their connectivity  with existing nodes in the network during the testing phase.  All algorithms are carried out on the training set of $M$ nodes, and the obtained model is  used to estimate the function value on the  newly arriving nodes. The runtime for estimating the function value on the newly-joining nodes, as well as the  generalization ${\rm NMSE}:=\frac{1}{|{\cal S}^c|}\|\hat{\bbx}_{{\cal S}^c}-\bbx_{{\cal S}^c}\|_2^2/\|\bbx_{{\cal S}^c}\|_2^2$ performance is evaluated, with ${\cal S}^c$ denoting the index set of new nodes. The Gradraker \textcolor{black}{adopts a dictionary consisting of $2$ Gaussian kernels with parameters $\sigma^2=1,5$, using $D=10$ random features, and it is compared with: a) the $k$NN algorithm, with $k$ selected as the maximum number of neighbors a node has in a specific network, and with the combining weights set to $1/k$ in unweighted graphs, and $a_{il}/\sum_{j\in\mathcal{N}_i}a_{ij}$ for the $l$th neighbor in weighted graphs;} b) the graph kernel (GK) based method using  diffusion kernels with different bandwidths (named as GK-DF), or band-limited kernels with different bandwidths (GK-BL); and c) kernel based learning without RF approximation (KL) with a Gaussian kernel of $\sigma^2=5$. Results are averaged over $100$ independent runs. The regularization parameter for all algorithms is selected from the set $\mu=\{10^{-7},10^{-6}, \dots, 10^0\}$ via cross validation.
\noindent\textbf{Testing results.} 
 Figure \ref{fig1}  illustrates the performance in terms of the average runtime and NMSE versus the number of nodes (size) of the network. 
 In this experiment, $\bar{\bbK}$ in \eqref{eq:gk} is  generated from the normalized graph Laplacian  $\bbL$,  using the diffusion kernel $r(\lambda)=\exp (\sigma^2\lambda/2)$. A bandwidth of $\sigma^2=5$ was used to generate the data.
 It is observed that GK attains the best  generalization accuracy when the ground-truth model is known, but its computational complexity grows rapidly with the network size. However, GK does not perform as well  when a mismatched kernel is applied. The Gradraker method on the other hand, is very efficient, while at the same time it can provide reasonable estimates of the signal on the newly arriving nodes, even without knowledge about the kernels. The k-NN method is very efficient, but does not provide as reliable performance as the Gradraker.
 
 Figure \ref{fig2} depicts the performance of competitive algorithms. Matrix $\bar{\bbK}$ for data generation is formed based on \eqref{eq:gka} using the Gaussian kernel $\kappa(\bba_i-\bba_j)= \exp(\|\bba_i-\bba_j\|^2/\sigma^2)$, with $\sigma^2=5$. In this case, KL exactly matches the true model, and hence it achieves the best performance. However, it is  the most complex in terms of runtime. Meanwhile, GK-based methods suffer from model mismatch, and are  also relatively more complex than Graderaker. The novel Gradraker is capable of estimating the nodal function on the newly joining nodes with high accuracy at very low computational complexity.
 Note that in real-world scenarios, accurate prior information about the underlying model is often unavailable, in which case  Gradraker can be a more reliable and efficient choice. 

\begin{figure*}[t]
\centering
		\begin{minipage}[b]{.49\textwidth}
		\centering
		\includegraphics[width=9cm]{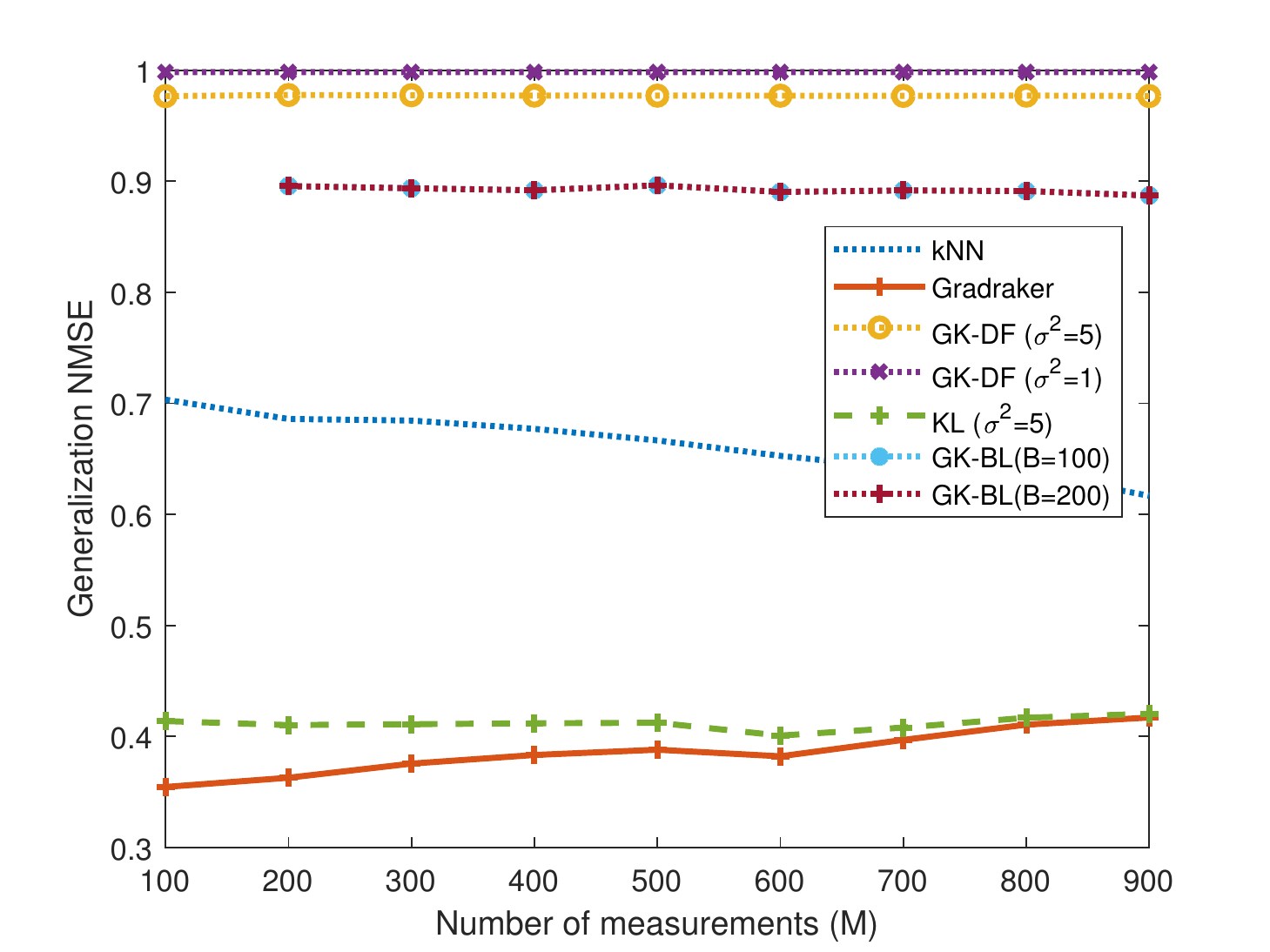}
		\centerline{(a) Generalization NMSE}
	\end{minipage}
	\begin{minipage}[b]{.49\textwidth}
		\centering
		\includegraphics[width=9cm]{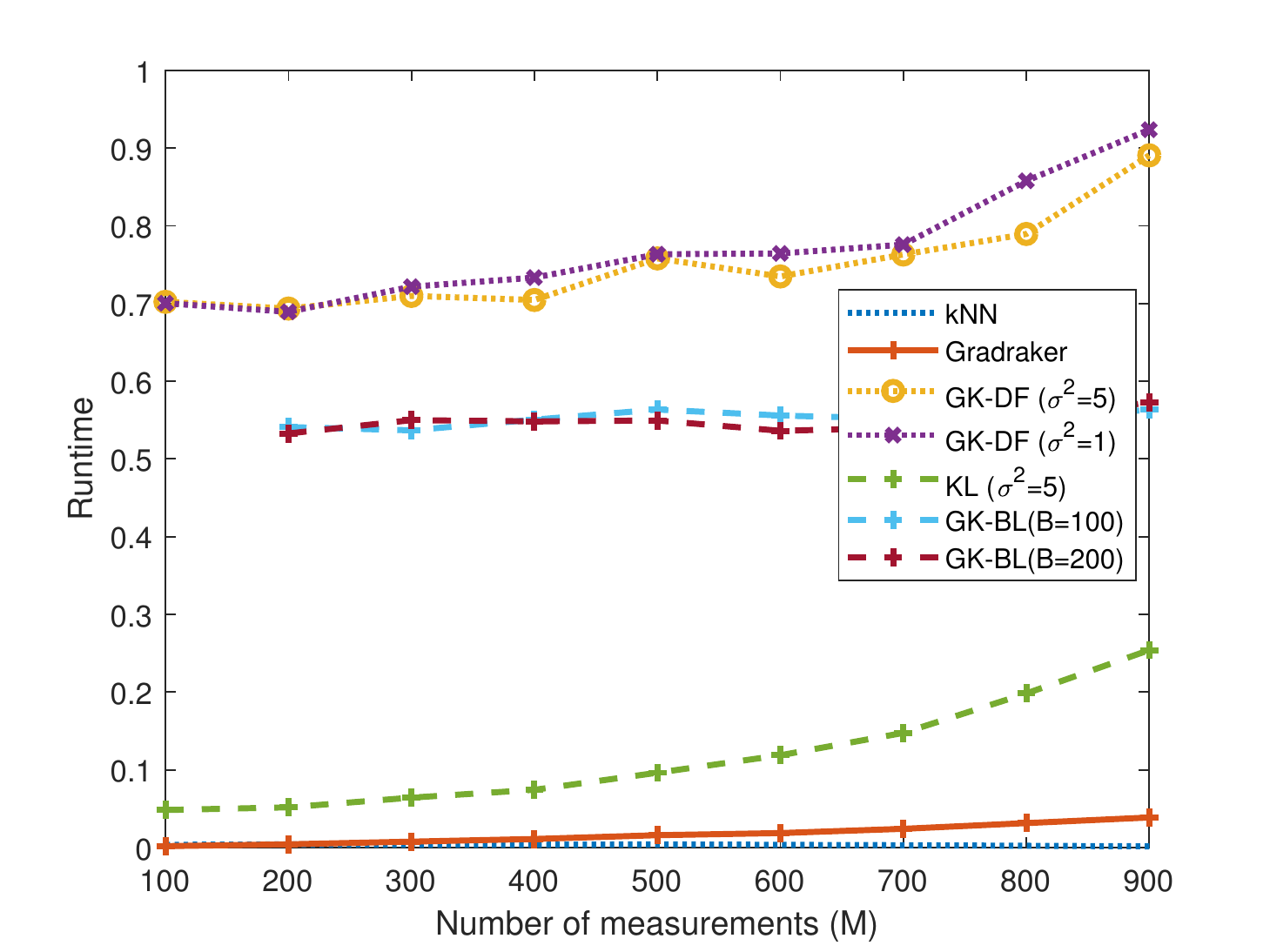}
		\centerline{(b) Runtime  }
	\end{minipage}
		\caption{{Inference performance versus number of sampled nodes in email dataset} }\label{fig:email}
\end{figure*}

\subsection{Reconstruction of the temperature data}
This subsection tests the performance of Gradraker on a real temperature dataset. The dataset comprises $24$ signals corresponding to the average 
temperature per month in the intervals $1961-1980$ and $1991-2010$ 
measured by $89$ stations in Switzerland {\cite{tempdata}}. The training set contains the first $12$ signals, corresponding to the interval $1961-1980$, while the test set contains the remaining $12$. Each station is represented by a node, and the graph was constructed using the algorithm in \cite{dong2016learning} based on the training signals. Given the test signal on a randomly chosen subset of $M$ vertices, the values at the remaining $N-M$ vertices are estimated as newly-coming nodes. \textcolor{black}{The generalization NMSE over the $N-M$ nodes is averaged across the test signals.}

 Fig. \ref{fig:temp} compares
the performance of Gradraker with those of competing alternatives. Gradraker adopts a dictionary consisting of $3$ Gaussian kernels with parameters $\sigma^2=1,5,10$, using $D=100$ random features.  
It is clear from Fig. \ref{fig:temp} that Gradraker outperforms GK in both generalization NMSE and runtime. On the other hand, even though KL achieves  lower generalization NMSE, it incurs a much higher complexity.

\begin{figure*}[t]
\centering
		\begin{minipage}[b]{.49\textwidth}
		\centering
		\includegraphics[width=9cm]{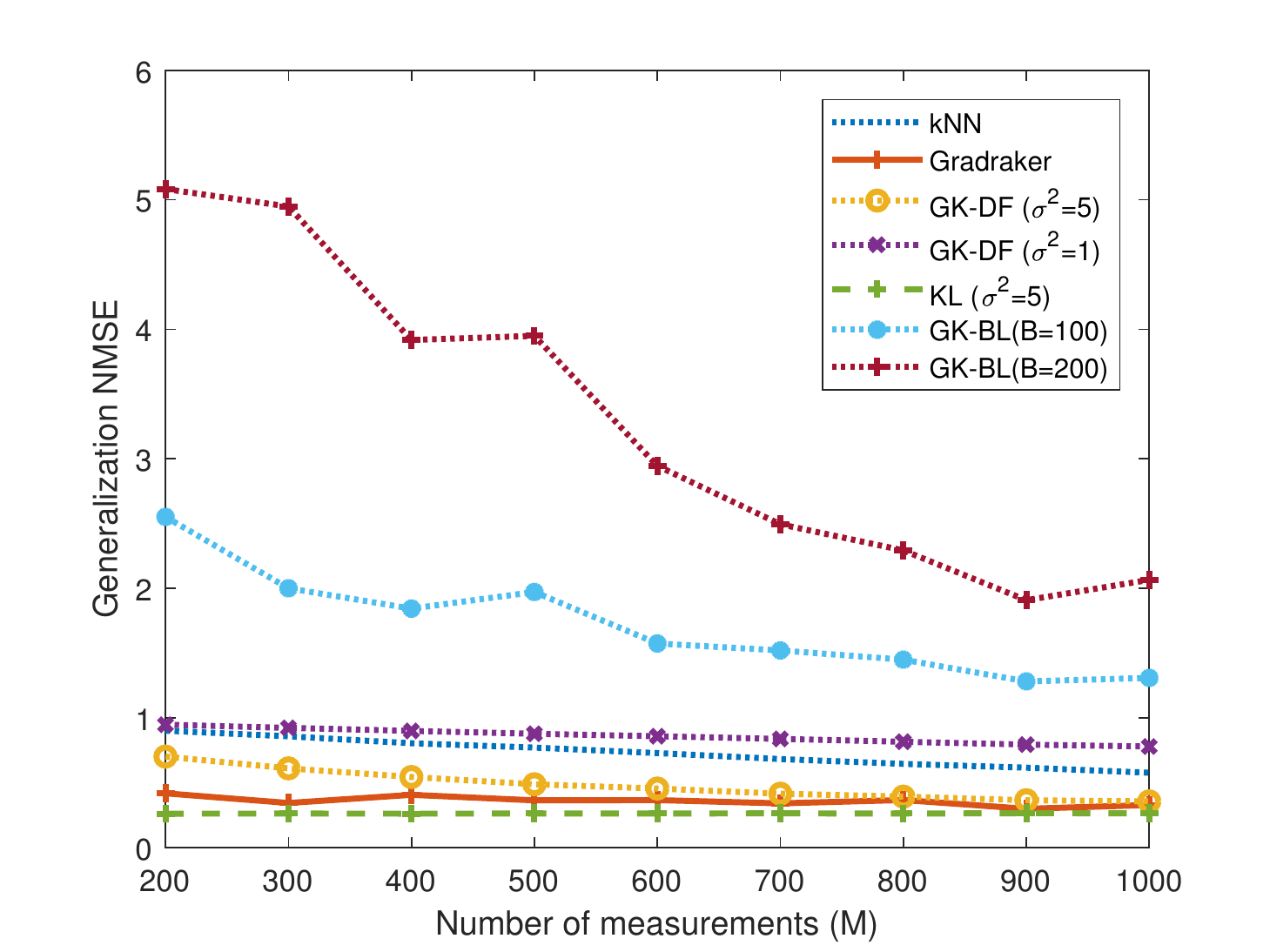}
		\centerline{(a) Generalization NMSE}
	\end{minipage}
	\begin{minipage}[b]{.49\textwidth}
		\centering
		\includegraphics[width=9cm]{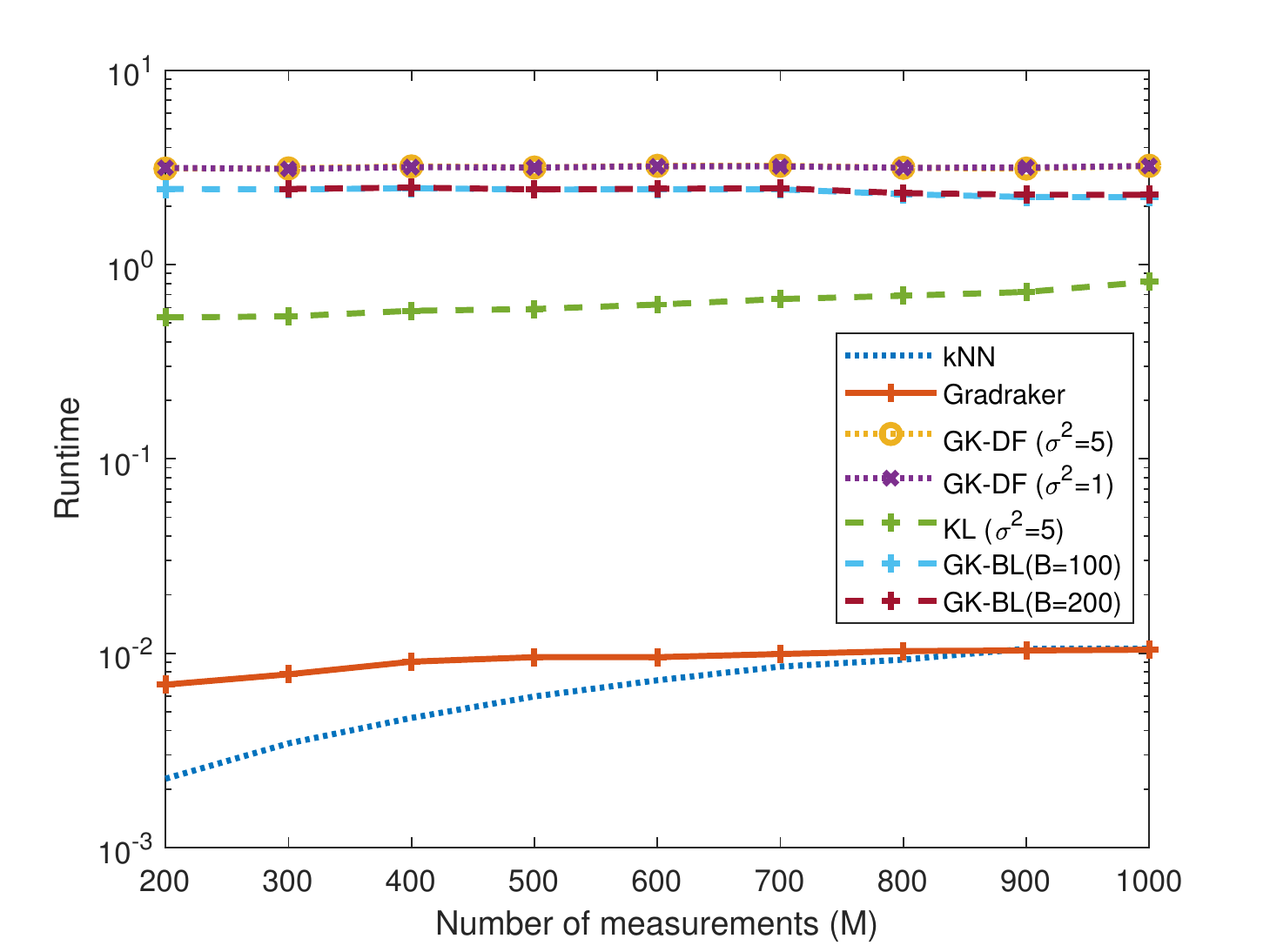}
		\centerline{(b)  Runtime  }
	\end{minipage}
		\caption{{Inference performance versus number of sampled nodes in Cora dataset} }\label{fig:cora}
\end{figure*}

\subsection{Reconstruction of the Email-Eu-core data} 
 The Eu-core network was generated using email data from a large European research institution \cite{leskovec2007tkdd}, where each node represents a person, and  an edge $(i,j)$ is present if person $i$ sent person $j$ at least one email. The e-mails only represent communication between institution members (the core), and the dataset does not contain incoming messages from or outgoing messages to the rest of the world. The dataset also contains ``ground-truth'' community memberships of the nodes. Each individual belongs to one of 42 departments at the research institute. During the experiment, the department labels are considered to be $y_n$ that are to be sampled and estimated. The graph consists of $N=1,005$ nodes, and $25,571$ edges. Gradraker adopts a dictionary consisting of $2$ Gaussian kernels with parameters $\sigma^2=1,10$, from which  $D=10$ random features are generated. The test results were averaged over $100$ independent runs with randomly sampled nodes. 

%\begin{figure*}[t]
%	\begin{minipage}[b]{.49\textwidth}
%		\centering
%		\includegraphics[width=7.5cm]{figs/030818/email_acc.pdf}
%		\centerline{(a) Classification accuracy}
%	\end{minipage}
%	\begin{minipage}[b]{.49\textwidth}
%		\centering
%		\includegraphics[width=7.5cm]{figs/030818/email_rt.pdf}
%		\centerline{(b)  Runtime for Eu-email network }
%	\end{minipage}
%		\caption{{Classification performance vs. number of sampled nodes: (a) Classification accuracy;  (b) Runtime.  } }
%	\label{fig:email}
%%	\vspace{-1cm}
%\end{figure*}

  Fig. \ref{fig:email} compares
the performance of Gradraker with those of alternative algorithms when different numbers of nodal labels are observed. 
%Gradraker adopts a Gaussian kernel with parameter $\sigma^2=5$ with $100$ random features. GK uses diffusion kernels for different values of $\sigma^2$. KL method adopts Gaussian kernel with $\sigma^2=5$. 
% 
 It is clear that the RF-based approach outperforms the GK-based method in both reconstruction accuracy and runtime. While the batch KL method without RF approximation outperforms the RF method, it incurs considerably higher computational complexity.
 
\subsection{Reconstruction of the Cora data}
This subsection tests the Gradraker algorithm on the Cora citation dataset \cite{lu2003link}. Gradraker adopts a dictionary consisting of $2$ Gaussian kernels with parameters $\sigma^2=1,10$, using $D=20$ random features. The results were averaged over $100$ independent runs.
The Cora dataset consists of $2,708$ scientific publications classified into one of seven classes. The citation network consists of $5,429$ links. The network is constructed so that  a link connects node $i$ to node $j$ if paper $i$ cites paper $j$, and the category id the paper belongs to is to be reconstructed. It can be observed again from Figure \ref{fig:cora}, that the Gradraker markedly outperforms the GK algorithms in terms of generalization NMSE, and is much more computationally efficient than all other algorithms except the kNN method, which however does not perform as well. 

It can be readily observed from our numerical results over synthetic and real datasets,  that the Gradraker provides reliable performance in terms of NMSE in all tests, while at the same time, it scales much better than all kernel based alternatives. This is because the  alternative kernel-based algorithms require re-computing the kernel matrix whenever a new node joins the network. It is worth noting that all kernel-based alternatives require exact knowledge of the entire network topology, which is not necessary for GradRaker that only requires $\{\bbz_{\bbV}(\bba_n)\}$. These tests corroborate the potential of GradRaker for application settings, where the graphs grow and nodes have privacy constraints.

\section{Conclusions}\label{sec:con}
The present paper deals with the problem of reconstructing signals over graphs, from samples over a subset of nodes. An online MKL based algorithm is developed, which is capable of estimating and updating the nodal functions even when samples are collected sequentially. The novel online scheme is highly scalable and can estimate the unknown signals on newly joining nodes. Unlike many existing approaches, it only relies on encrypted nodal connectivity information, which is appealing for networks where nodes have strict privacy constraints. 

This work opens up a number of interesting directions for future research, including:  a) exploring distributed
implementations that are well motivated in large-scale networks; b) graph-adaptive learning  when multiple sets of features are available; and c) developing adaptive sampling strategies for Gradraker.

%
%\bibliographystyle{plainnat}
%\bibliography{myabrv,net,dmkl,bib4graph}

\begin{thebibliography}{10}
\providecommand{\url}[1]{#1}
\csname url@samestyle\endcsname
\providecommand{\newblock}{\relax}
\providecommand{\bibinfo}[2]{#2}
\providecommand{\BIBentrySTDinterwordspacing}{\spaceskip=0pt\relax}
\providecommand{\BIBentryALTinterwordstretchfactor}{4}
\providecommand{\BIBentryALTinterwordspacing}{\spaceskip=\fontdimen2\font plus
\BIBentryALTinterwordstretchfactor\fontdimen3\font minus
  \fontdimen4\font\relax}
\providecommand{\BIBforeignlanguage}[2]{{%
\expandafter\ifx\csname l@#1\endcsname\relax
\typeout{** WARNING: IEEEtran.bst: No hyphenation pattern has been}%
\typeout{** loaded for the language `#1'. Using the pattern for}%
\typeout{** the default language instead.}%
\else
\language=\csname l@#1\endcsname
\fi
#2}}
\providecommand{\BIBdecl}{\relax}
\BIBdecl

\bibitem{kolaczyk2009statistical}
E.~D. Kolaczyk, \emph{Statistical Analysis of Network Data: Methods and
  Models}.\hskip 1em plus 0.5em minus 0.4em\relax Springer, 2009.

\bibitem{kondor2002diffusion}
R.~I. Kondor and J.~Lafferty, ``Diffusion kernels on graphs and other discrete
  structures,'' in \emph{Proc. Intl. Conf. on Machine Learning}, Sydney,
  Australia, Jul. 2002, pp. 315--322.

\bibitem{belkin2006manifold}
M.~Belkin, P.~Niyogi, and V.~Sindhwani, ``Manifold regularization: A geometric
  framework for learning from labeled and unlabeled examples,'' \emph{J. of
  Mach. Learn. Res.}, vol.~7, pp. 2399--2434, Nov. 2006.

\bibitem{wasserman2008statistical}
L.~Wasserman and J.~D. Lafferty, ``Statistical analysis of semi-supervised
  regression,'' in \emph{Advances in Neural Information Processing Systems},
  Vancouver, Canada, 2008, pp. 801--808.

\bibitem{lu2003link}
Q.~Lu and L.~Getoor, ``Link-based classification,'' in \emph{Proc. of Intl.
  Conf. on Machine Learning}, Washington DC, USA, 2003, pp. 496--503.

\bibitem{giannakis2018pieee}
G.~B. Giannakis, Y.~Shen, and G.~V. Karanikolas, ``Topology identification and
  learning over graphs: Accounting for nonlinearities and dynamics,''
  \emph{Proc. of the IEEE}, vol. 106, no.~5, pp. 787--807, May 2018.

\bibitem{chapelle2009semi}
O.~Chapelle, B.~Scholkopf, and A.~Zien, ``Semi-supervised learning,''
  \emph{IEEE Trans. Neural Netw.}, vol.~20, no.~3, pp. 542--542, Mar. 2009.

\bibitem{chapelle2000transductive}
O.~Chapelle, V.~Vapnik, and J.~Weston, ``Transductive inference for estimating
  values of functions,'' in \emph{Advances in Neural Information Processing
  Systems}, Denver, CO, USA, 1999, pp. 421--427.

\bibitem{cortes2007transductive}
C.~Cortes and M.~Mohri, ``On transductive regression,'' in \emph{Advances in
  Neural Inf. Process. Syst.}, Vancouver, Canada, Dec. 2006, pp. 305--312.

\bibitem{berberidis2018adaptive}
D.~Berberidis, A.~N. Nikolakopoulos, and G.~B. Giannakis, ``Adaptive diffusions
  for scalable learning over graphs,'' \emph{arXiv preprint arXiv:1804.02081},
  2018.

\bibitem{narang2013signal}
S.~K. Narang, A.~Gadde, and A.~Ortega, ``Signal processing techniques for
  interpolation in graph structured data,'' in \emph{Proc IEEE Intl. Conf.
  Acoust. Speech Signal Process.}, Vancouver, Canada, 2013, pp. 5445--5449.

\bibitem{wang2015local}
X.~Wang, P.~Liu, and Y.~Gu, ``Local-set-based graph signal reconstruction,''
  \emph{IEEE Trans. Signal Process.}, vol.~63, no.~9, pp. 2432--2444, May 2015.

\bibitem{romero2017kernel}
D.~Romero, M.~Ma, and G.~B. Giannakis, ``Kernel-based reconstruction of graph
  signals,'' \emph{IEEE Trans. on Sig. Process.}, vol.~65, no.~3, pp. 764--778,
  Feb. 2017.

\bibitem{marques2016sampling}
A.~G. Marques, S.~Segarra, G.~Leus, and A.~Ribeiro, ``Sampling of graph signals
  with successive local aggregations,'' \emph{IEEE Transactions on Signal
  Processing}, vol.~64, no.~7, pp. 1832--1843, April 2016.

\bibitem{shuman2013emerging}
D.~I. Shuman, S.~K. Narang, P.~Frossard, A.~Ortega, and P.~Vandergheynst, ``The
  emerging field of signal processing on graphs: Extending high-dimensional
  data analysis to networks and other irregular domains,'' \emph{IEEE Sig.
  Process. Mag.}, vol.~30, no.~3, pp. 83--98, 2013.

\bibitem{smola2003kernels}
A.~J. Smola and R.~I. Kondor, ``Kernels and regularization on graphs,'' in
  \emph{Learning Theory and Kernel Machines}.\hskip 1em plus 0.5em minus
  0.4em\relax Springer, 2003, pp. 144--158.

\bibitem{ioannidis2018inference}
V.~N. Ioannidis, D.~Romero, and G.~B. Giannakis, ``Inference of spatio-temporal
  functions over graphs via multikernel kriged {K}alman filtering,'' \emph{IEEE
  Trans. Signal Process.}, vol.~66, no.~12, pp. 3228--3239, 2018.

\bibitem{wahba1990spline}
G.~Wahba, \emph{Spline {M}odels for {O}bservational {D}ata}.\hskip 1em plus
  0.5em minus 0.4em\relax Philadelphia, PA: SIAM, 1990.

\bibitem{di2018adaptive}
P.~Di~Lorenzo, P.~Banelli, E.~Isufi, S.~Barbarossa, and G.~Leus, ``Adaptive
  graph signal processing: Algorithms and optimal sampling strategies,''
  \emph{IEEE Trans. Sig. Process.}, vol.~66, no.~13, pp. 3584 -- 3598, July
  2018.

\bibitem{altman1992introduction}
N.~S. Altman, ``An introduction to kernel and nearest-neighbor nonparametric
  regression,'' \emph{The American Statistician}, vol.~46, no.~3, pp. 175--185,
  Feb. 1992.

\bibitem{rahimi2007}
A.~Rahimi and B.~Recht, ``Random features for large-scale kernel machines,'' in
  \emph{Proc. Advances in Neural Info. Process. Syst.}, Vancouver, Canada, Dec.
  2007, pp. 1177--1184.

\bibitem{shen2018aistats}
Y.~Shen, T.~Chen, and G.~B. Giannakis, ``Online ensemble multi-kernel learning
  adaptive to non-stationary and adversarial environments,'' in \emph{Proc. of
  Intl. Conf. on Artificial Intelligence and Statistics}, Lanzarote, Canary
  Islands, Apr. 2018.

\bibitem{hazan2016}
E.~Hazan, ``Introduction to online convex optimization,'' \emph{Found. and
  Trends in Mach. Learn.}, vol.~2, no. 3-4, pp. 157--325, 2016.

\bibitem{cortes2009}
C.~Cortes, M.~Mohri, and A.~Rostamizadeh, ``$\ell_2$-regularization for
  learning kernels,'' in \emph{Proc. Conf. on Uncertainty in Artificial
  Intelligence}, Montreal, Canada, Jun. 2009, pp. 109--116.

\bibitem{chen2009fast}
J.~Chen, H.~Fang, and Y.~Saad, ``Fast approximate k{NN} graph construction for
  high dimensional data via recursive {L}anczos bisection,'' \emph{Journal of
  Machine Learning Research}, vol.~10, pp. 1989--2012, September 2009.

\bibitem{kivela2014multilayer}
M.~Kivel{\"a}, A.~Arenas, M.~Barthelemy, J.~P. Gleeson, Y.~Moreno, and M.~A.
  Porter, ``Multilayer networks,'' \emph{Journal of {Complex Networks}},
  vol.~2, no.~3, pp. 203--271, 2014.

\bibitem{traganitis2017}
P.~Traganitis, Y.~Shen, and G.~B. Giannakis, ``Topology inference for
  multilayer networks,'' in \emph{Intl. Workshop on Network Science for
  Comms.}, Atlanta, GA, May 2017.

\bibitem{zhou2007co}
D.~Zhou, S.~A. Orshanskiy, H.~Zha, and C.~L. Giles, ``Co-ranking authors and
  documents in a heterogeneous network,'' in \emph{Proc. of Intl. Conf. on Data
  Mining}, Omaha NE, USA, October 2007, pp. 739--744.

\bibitem{sun2013mining}
Y.~Sun and J.~Han, ``Mining heterogeneous information networks: {A} structural
  analysis approach,'' \emph{ACM SIGKDD Explorations Newsletter}, vol.~14,
  no.~2, pp. 20--28, 2013.

\bibitem{shalev2011}
S.~Shalev-Shwartz, ``Online learning and online convex optimization,''
  \emph{Found. and Trends in Mach. Learn.}, vol.~4, no.~2, pp. 107--194, 2011.

\bibitem{micchelli2005}
C.~A. Micchelli and M.~Pontil, ``Learning the kernel function via
  regularization,'' \emph{J. Mach. Learn. Res.}, vol.~6, pp. 1099--1125, Jul.
  2005.

\bibitem{lu2016large}
J.~Lu, S.~C. Hoi, J.~Wang, P.~Zhao, and Z.-Y. Liu, ``Large scale online kernel
  learning,'' \emph{J. Mach. Learn. Res.}, vol.~17, no.~47, pp. 1--43, Apr.
  2016.

\bibitem{erdos1959random}
P.~Erdos and A.~R{\'e}nyi, ``On random graphs {I},'' \emph{Publ. Math.
  Debrecen}, vol.~6, pp. 290--297, 1959.

\bibitem{tempdata}
\BIBentryALTinterwordspacing
``Meteorology and climatology meteoswiss.'' [Online]. Available:
  \url{http://www.meteoswiss.admin.ch/home/climate/past/
  climate-normals/climate-diagrams-and-normal-values-per-station.html}
\BIBentrySTDinterwordspacing

\bibitem{dong2016learning}
X.~Dong, D.~Thanou, P.~Frossard, and P.~Vandergheynst, ``Learning {L}aplacian
  matrix in smooth graph signal representations,'' \emph{IEEE Trans. on Sig.
  Process.}, vol.~64, no.~23, pp. 6160--6173, Dec. 2016.

\bibitem{leskovec2007tkdd}
J.~Leskovec, J.~Kleinberg, and C.~Faloutsos, ``Graph evolution: {D}ensification
  and shrinking diameters,'' \emph{ACM Transactions on Knowledge Discovery from
  Data}, vol.~1, no.~1, Mar. 2007.

\end{thebibliography}

%\clearpage

\appendices

\section{Proof of Lemma \ref{lemma4}}\label{app.pf.lemma4}
To prove Lemma \ref{lemma4}, we introduce two intermediate lemmata.  
\begin{lemma}\label{lemma3}
Under (as1), (as2), and $\hat{f}_p^*$ as in \eqref{eq.slot-opt} with ${\cal F}_p:=\{\hat{f}_p|\hat{f}_p(\bba)=\bbtheta^{\top}\mathbf{z}_p(\bba),\,\forall \bbtheta\in\mathbb{R}^{2D}\}$, let $\{\hat{f}_{p,t}(\bba_t)\}$ denote the sequence of estimates generated by Gradraker with a pre-selected kernel $\kappa_p$. Then the following bound holds true w.p.1
	\begin{align}
	\sum_{t=1}^T {\cal L}_t(\hat{f}_{p,t}(\bba_t))\!-\!\sum_{t=1}^T{\cal L}_t(\hat{f}_p^*(\bba_t))\!\leq\! \frac{\|\bbtheta_p^*\|^2}{2\eta}\!+\!\frac{\eta L^2T}{2}
	\end{align}
	where $\eta$ is the learning rate, $L$ is the Lipschitz constant in (as2), and $\bbtheta_p^*$ is the corresponding parameter (or weight) vector supporting the best estimator  $\hat{f}_p^*(\bba)=(\bbtheta_p^*)^{\top}\mathbf{z}_p(\bba)$.
\end{lemma}
%\begin{proof}
%	See Appendix \ref{app.pf.lemma3}.
%\end{proof}

\begin{proof}
The proof is similar to the regret analysis of online gradient descent, see e.g., \cite{shen2018aistats}.
\end{proof}
%Lemma \ref{lemma3} establishes that the static regret of the Gradraker is upper bounded by some constants, which mainly depend on the stepsize in \eqref{eq.klp-weight} and the time horizon $T$.

In addition, we will bound the difference between the loss of the solution obtained from Algorithm \ref{algo:omkl:rf} and the loss of the best single kernel-based online learning algorithm. Specifically, the following lemma holds.
\begin{lemma}
\label{lemma10}
Under (as1) and (as2), with $\{\hat{f}_{p,t}\}$ generated from Gradraker, it holds that 
	\begin{equation}
	\label{eq:lemm10}
		\sum_{t=1}^T \sum_{p=1}^P \bar{w}_{p,t} {\cal L}_{t}(\hat{f}_{p,t}(\bba_t))- \sum_{t=1}^T {\cal L}_{t}(\hat{f}_{p,t}(\bba_t))\leq\eta T+\frac{\ln P}{\eta}
	\end{equation}
	where $\eta$ is the learning rate in \eqref{eq.mkl-weight}, and $P$ is the number of kernels in the dictionary. 
\end{lemma}

\begin{proof}
Letting $W_{t}:=\sum_{p=1}^P w_{p,t}$, the weight recursion in \eqref{eq.mkl-weight} implies that
\begin{eqnarray}
\label{eq:sreg:W1}
	W_{t+1}&=&\!\!\!\sum_{p=1}^P w_{p,t+1}=\sum_{p=1}^P w_{p,t} \exp\left(-\eta{\cal L}_t\left(\hat{f}_{p,t}(\bba_t)\right)\right)\\
	&\leq&\!\!\!\sum_{p=1}^P w_{p,t}\left(1-\eta{\cal L}_t\left(\hat{f}_{p,t}(\bba_t)\right)+\eta^2{\cal L}_t\left(\hat{f}_{p,t}(\bba_t)\right)^2\right)\nonumber
\end{eqnarray}
where the last inequality holds because $\exp(-\eta x)\leq 1-\eta x+\eta^2 x^2$, for $|\eta|\leq 1$.
Furthermore, substituting $\bar{w}_{p,t}:=w_{p,t}/\sum_{p=1}^P w_{p,t}=w_{p,t}/W_t$ into \eqref{eq:sreg:W1} leads to
\begin{align}
	\label{eq:sreg:W2-0}
	W_{t+1}&\leq \sum_{p=1}^P W_t\bar{w}_{p,t}\!\left(\!1-\eta{\cal L}_t\left(\hat{f}_{p,t}(\bba_t)\right)\!+\!\eta^2{\cal L}_t\left(\hat{f}_{p,t}(\bba_t)\right)^2\!\right)\nonumber\\
	&= W_t\Bigg(1-\eta\sum_{p=1}^P\bar{w}_{p,t} {\cal L}_t\left(\hat{f}_{p,t}(\bba_t)\right)\nonumber\\
	&\hspace{2cm}+\eta^2\sum_{p=1}^P\bar{w}_{p,t} {\cal L}_t\left(\hat{f}_{p,t}(\bba_t)\right)^2\Bigg).
\end{align}
Since $1+x\leq e^x,\,\forall x$, it follows that
\begin{align}
	\label{eq:sreg:W2}
	W_{t+1}\leq& W_t \exp \Bigg(-\eta \sum_{p=1}^P\bar{w}_{p,t} {\cal L}_t\left(\hat{f}_{p,t}(\bba_t)\right)\nonumber\\
	&+\eta^2 \sum_{p=1}^P\bar{w}_{p,t} {\cal L}_t\left(\hat{f}_{p,t}(\bba_t)\right)^2\Bigg).
\end{align}
Telescoping \eqref{eq:sreg:W2} from $t=1$ to $T$ yields
\begin{align}\label{eq:sreg:W3}
	W_{T+1}\leq& \exp \Bigg(-\eta \sum_{t=1}^T\sum_{p=1}^P\bar{w}_{p,t} {\cal L}_t\left(\hat{f}_{p,t}(\bba_t)\right)\nonumber\\
	&+\eta^2  \sum_{t=1}^T\sum_{p=1}^P\bar{w}_{p,t} {\cal L}_t\left(\hat{f}_{p,t}(\bba_t)\right)^2\Bigg).
\end{align}

On the other hand, for any $p$, it holds that 
\begin{eqnarray}
\label{eq:sreg:W4}
	W_{T+1}&\geq & w_{p,T+1}\nonumber\\
	&=&w_{p,1}\prod_{t=1}^T \exp(-\eta{\cal L}_t\left(\hat{f}_{p,t}(\bba_t)\right))\nonumber\\
	&= & w_{p,1}\exp\Bigg(-\eta\sum_{t=1}^T{\cal L}_t\left(\hat{f}_{p,t}(\bba_t)\right)\Bigg).
\end{eqnarray}
Combining \eqref{eq:sreg:W3} with \eqref{eq:sreg:W4}, we arrive at
\begin{align}
\label{eq:sreg:6} 
&\exp \Bigg(\!-\!\eta \sum_{t=1}^T\sum_{p=1}^P\bar{w}_{p,t} {\cal L}_t\left(\hat{f}_{p,t}(\bba_t)\right)\nonumber\\
&\hspace{2cm}+\eta^2  \sum_{t=1}^T\sum_{p=1}^P\bar{w}_{p,t} {\cal L}_t\left(\hat{f}_{p,t}(\bba_t)\right)^2\!\Bigg)\nonumber\\
&	\geq\,  w_{p,1} \exp\Bigg(\!-\!\eta\sum_{t=1}^T{\cal L}_t\left(\hat{f}_{p,t}(\bba_t)\right)\!\Bigg).
\end{align}
Taking the logarithm on both sides of \eqref{eq:sreg:6}, and recalling  that $w_{p,1}=1/P$, we obtain
\begin{align}
\label{eq:sreg:7}
&-\eta \sum_{t=1}^T\sum_{p=1}^P\bar{w}_{p,t} {\cal L}_t\!\left(\hat{f}_{p,t}(\bba_t)\right)\!+\eta^2  \sum_{t=1}^T\sum_{p=1}^P\bar{w}_{p,t} {\cal L}_t\!\left(\hat{f}_{p,t}(\bba_t)\right)^2\nonumber\\
	\!\geq\!&-\eta\sum_{t=1}^T{\cal L}_t\!\left(\hat{f}_{p,t}(\bba_t)\right)\!-\ln P.
\end{align}
Re-organizing the terms leads to
\begin{align}\label{eq:sreg:8}
	&\sum_{t=1}^T\sum_{p=1}^P\bar{w}_{p,t} {\cal L}_t\left(\hat{f}_{p,t}(\bba_t)\right)\\
	\leq &\sum_{t=1}^T{\cal L}_t\left(\hat{f}_{p,t}(\bba_t)\right)+\eta  \sum_{t=1}^T\sum_{p=1}^P\bar{w}_{p,t} {\cal L}_t\left(\hat{f}_{p,t}(\bba_t)\right)^2+\frac{\ln P}{\eta}\nonumber
	\end{align}
and the proof is complete, since ${\cal L}_t\left(\hat{f}_{p,t}(\bba_t)\right)^2\leq 1$ and $\sum_{p=1}^P\bar{w}_{p,t}=1$.
\end{proof} 

Since ${\cal L}_t(\cdot)$ is convex under (as1), Jensen's inequality implies 
\begin{align}
\label{eq:sreg:9}
{\cal L}_t\bigg(\sum_{p=1}^P \bar{w}_{p,t} \hat{f}_{p,t}(\bba_t)\bigg)\leq \sum_{p=1}^P \bar{w}_{p,t} {\cal L}_t\left(\hat{f}_{p,t}(\bba_t)\right).
\end{align}
Combining \eqref{eq:sreg:9} with  Lemma \ref{lemma10}, one arrives readily at 
	\begin{align}
		&\sum_{t=1}^T{\cal L}_t\bigg(\sum_{p=1}^P \bar{w}_{p,t} \hat{f}_{p,t}(\bba_t)\bigg)\nonumber\\
		\leq &\sum_{t=1}^T {\cal L}_{t}\left(\hat{f}_{p,t}(\bba_t)\right)+\eta T+\frac{\ln P}{\eta}\nonumber\\
	\stackrel{(a)}{\leq}  &\sum_{t=1}^T{\cal L}_t\left(\hat{f}_p^*(\bba_t)\right)+\frac{\ln P}{\eta}+\frac{\|\bbtheta_p^*\|^2}{2\eta}+\frac{\eta L^2T}{2}+\eta T
	\end{align}
	%.
where (a) follows due to Lemma \ref{lemma3} and because $\bbtheta_p^*$ is the optimal solution for any given $\kappa_p$. This proves Lemma \ref{lemma4}.

\section{Proof of Theorem \ref{theorem0}}\label{app.pf.theorem0}
To bound the performance relative to the best estimator $f^*(\bba_t)$ in the RKHS, the key step is to bound the approximation error. 
For a given shift-invariant $\kappa_p$, the maximum point-wise error of the RF kernel approximant is  bounded with probability at least
$
	1-2^8\big(\frac{\sigma_p}{\epsilon}\big)^2 \exp \big(\frac{-D\epsilon^2}{4N+8}\big), 
$ by \cite{rahimi2007}
\begin{align}
\label{ieq:1}
	\sup_{\bba_i,\bba_j\in{\cal X}} \left|\bbz_p^\top(\bba_i)\bbz_p(\bba_j)-\kappa_p(\bba_i,\bba_j)\right|<\epsilon 
\end{align}
where $\epsilon>0$ is a given constant, $D$ the number of features, while $M$ is the number of nodes already in the network, and $\sigma_p^2:=\mathbb{E}_p[\|\bbv\|^2]$ is the second-order moment of the RF vector norm induced by $\kappa_p$.  
Henceforth, for the optimal function estimator \eqref{eq.slot-opt} in ${\cal H}_p$ denoted by $f_p^*(\bba):=\sum_{t=1}^T\alpha_{p,t}^* \kappa_p(\bba,\bba_t)$, and its RF-based approximant $\check{f}^*_p:=\sum_{t=1}^T\alpha_{p,t}^* \bbz_p^\top(\bba)\bbz_p(\bba_t)\in{\cal F}_p$, we have
\begin{align}
\label{eq:sreg:5.a}
	&\left|\sum_{t=1}^T{\cal L}_t\left(\check{f}^*_p(\bba_t)\right)-\sum_{t=1}^T{\cal L}_t\left(f_p^*(\bba_t)\right)\right|\nonumber\\
\stackrel{(a)}{\leq}\, &\sum_{t=1}^T\left|{\cal L}_t\left(\check{f}^*_p(\bba_t)\right)-{\cal L}_t(f_p^*(\bba_t))\right|\nonumber\\
\stackrel{(b)}{\leq}\, &\sum_{t=1}^T L\left|\sum_{t'=1}^T\alpha_{p,t'}^*\bbz_p^\top(\bba_{t'})\bbz_p(\bba_t)-\sum_{t'=1}^T\alpha_{p,t'}^*\kappa_p(\bba_{t'},\bba_t)\right|\nonumber\\
\stackrel{(c)}{\leq}\, &\sum_{t=1}^T L\sum_{t'=1}^T|\alpha_{p,t'}^*|\left|\bbz_p^\top(\bba_{t'})\bbz_p(\bba_t)-\kappa_p(\bba_{t'},\bba_t)\right|
\end{align}
where (a) is due to the triangle inequality; (b) uses the Lipschitz continuity of the loss, and (c) is due to the Cauchy-Schwarz inequality.  
Combining with \eqref{ieq:1}, yields
\begin{eqnarray}
\label{eq:sreg:5}
&&\left|\sum_{t=1}^T{\cal L}_t(\check{f}^*_p(\bba_t))-\sum_{t=1}^T{\cal L}_t(f_p^*(\bba_t))\right|\nonumber\\
&\leq &\sum_{t=1}^T L\epsilon\sum_{t'=1}^T |\alpha_{p,t'}^*|\leq \epsilon L T C,~{\rm w.h.p.}
\end{eqnarray}
where we used that $C:=\max_p \sum_{t=1}^T |\alpha_{p,t}^*|$.
Under the kernel bounds in (as3), the uniform convergence in \eqref{ieq:1} implies that $\sup_{\bba_t,\bba_{t'}\in{\cal X}} \bbz_p^\top(\bba_t)\bbz_p(\bba_{t'})\leq 1+\epsilon$, w.h.p., which  leads to
\begin{align}
\label{eq:sreg:4}
	\!\!\!\left\|\bbtheta_p^*\right\|^2:=&\left\|\sum_{t=1}^T\alpha_{p,t}^*\bbz_p(\bba_t)\right\|^2\nonumber\\
	=&\left|\sum_{t=1}^T\sum_{t'=1}^T\alpha_{p,t}^*\alpha_{p,t'}^*\bbz_p^{\top}(\bba_t)\bbz_p(\bba_{t'})\right|
	\leq (1+\epsilon)C^2
\end{align}
where for the last inequality we used the definition of $C$. 

Lemma \ref{lemma4} together with \eqref{eq:sreg:5} and \eqref{eq:sreg:4} lead to the regret of the proposed Gradraker algorithm relative to the best static function in ${\cal H}_p$, that is given by 
	\begin{align}
		&\sum_{t=1}^T{\cal L}_t\bigg(\sum_{p=1}^P w_{p,t} \hat{f}_{p,t}(\bba_t)\bigg)-\sum_{t=1}^T{\cal L}_t(f_p^*(\bba_t))\nonumber\\
		=&\sum_{t=1}^T{\cal L}_t\bigg(\sum_{p=1}^P w_{p,t} \hat{f}_{p,t}(\bba_t)\bigg)-\sum_{t=1}^T{\cal L}_t\left(\check{f}^*_p(\bba_t)\right)\nonumber\\
		&+\sum_{t=1}^T{\cal L}_t\left(\check{f}^*_p(\bba_t)\right)-\sum_{t=1}^T{\cal L}_t(f_p^*(\bba_t))\nonumber\\
		\leq\,  &\frac{\ln P}{\eta}+\frac{\eta L^2T}{2}+\eta T+ \frac{(1+\epsilon)C^2}{2\eta}+\epsilon L T C,~{\rm w.h.p.}
	\end{align}
which completes the proof of Theorem \ref{theorem0}.

%
%\section{Proof of Theorem \ref{theorem1}}\label{app.pf.theorem1
%}
%For the optimal function estimator \eqref{eq.slot-opt} in ${\cal H}^*$ denoted by $f^*(\bba):=\sum_{t=1}^T\alpha_{t}^* \kappa(\bba,\bba_t)$, and its RF-based approximant $\check{f}^*:=\sum_{t=1}^T\alpha_{t}^* \bbz^*(\bba)^\top\bbz^*(\bba_t)$, we have
%%
%\begin{align}
%\label{eq:sreg:6.a}
%	&\left|\sum_{t=1}^T{\cal L}_t\left(\check{f}^*_p(\bba_t)\right)-\sum_{t=1}^T{\cal L}_t\left(f^*(\bba_t)\right)\right|\nonumber\\
%\stackrel{(a)}{\leq}\, &\sum_{t=1}^T\left|{\cal L}_t\left(\check{f}^*_p(\bba_t)\right)-{\cal L}_t(f^*(\bba_t))\right|\nonumber\\
%\stackrel{(b)}{\leq}\, &\sum_{t=1}^T L\left|\sum_{t'=1}^T\alpha_{p,t'}^*\bbz_p^\top(\bba_{t'})\bbz_p(\bba_t)-\sum_{t'=1}^T\alpha_{t'}^*\kappa^*(\bba_{t'},\bba_t)\right|\nonumber\\
%\stackrel{(c)}{\leq}\, &\sum_{t=1}^T L\sum_{t'=1}^T|\alpha_{p,t'}^*|\left|\bbz_p^\top(\bba_{t'})\bbz_p(\bba_t)-\kappa_p(\bba_{t'},\bba_t)\right|
%\end{align}
%


\end{document}